\documentclass[conference]{IEEEtran}
\IEEEoverridecommandlockouts
\usepackage{cite}
\usepackage{amsmath,amssymb,amsfonts,amsthm}
\usepackage{algorithm}
\usepackage[noend]{algpseudocode}
\usepackage{graphicx}
\usepackage{textcomp}
\usepackage{mathrsfs,epsfig,epstopdf,tikz}
\usepackage{xcolor}
\def\BibTeX{{\rm B\kern-.05em{\sc i\kern-.025em b}\kern-.08em
    T\kern-.1667em\lower.7ex\hbox{E}\kern-.125emX}}
\begin{document}
\def \bw {\tilde{\mathbf{w}}}
\def \bx {\tilde{\mathbf{x}}}
\def \X {\mathcal{X}}
\def \Y {\mathcal{Y}}
\def \xx {\mathbf{x}}
\def \I {\mathbb{I}}
\def \R {\mathbb{R}}
\def \yy {\mathbf{y}}
\def \ww {\mathbf{w}}
\def \thetaa  {\mbox{\boldmath $\theta$}}
\def \lambdaa  {\mbox{\boldmath $\lambda$}}
\def \muu  {\mbox{\boldmath $\mu$}}
\def \etaa  {\mbox{\boldmath $\eta$}}
\def \deltaa  {\mbox{\boldmath $\delta$}}
\def \alphaa  {\mbox{\boldmath $\alpha$}}
\def \phii  {\mbox{\boldmath $\phi$}}
\def \I {\mathbb{I}}
\def \zero {\mathbf{0}}

\newtheorem{theorem}{Theorem}
\newtheorem{lemma}{Lemma}
\newtheorem{corr}{Corollary}

\title{Exact Passive-Aggressive Algorithms for Learning to Rank Using Interval Labels
}

\author{\IEEEauthorblockN{Naresh Manwani}
\IEEEauthorblockA{\textit{Machine Learning Lab} \\
\textit{IIIT Hyderabad, India} \\
nareshmanwani@gmail.com}
\and
\IEEEauthorblockN{Mohit Chandra}
\IEEEauthorblockA{\textit{Machine Learning Lab} \\
\textit{IIIT Hyderabad, India} \\
mohit.chandra@research.iiit.ac.in}
}

\maketitle

\begin{abstract}
 In this paper, we propose exact passive-aggressive (PA) online algorithms for learning to rank. The proposed algorithms can be used even when we have interval labels instead of actual labels for examples. The proposed algorithms solve a convex optimization problem at every trial. We find exact solution to those optimization problems to determine the updated parameters. We propose support class algorithm (SCA) which finds the active constraints using the KKT conditions of the optimization problems. These active constrains form support set which determines the set of thresholds that need to be updated. We derive update rules for PA, PA-I and PA-II. We show that the proposed algorithms maintain the ordering of the thresholds after every trial. We provide the mistake bounds of the proposed algorithms in both ideal and general settings. We also show experimentally that the proposed algorithms successfully learn accurate classifiers using interval labels as well as exact labels. Proposed algorithms also do well compared to other approaches. 
\end{abstract}

\begin{IEEEkeywords}
Ranking, online learning, passive-aggressive, interval labels, mistake bounds.
\end{IEEEkeywords}

\section{Introduction}
Ranking or ordinal regression is an important problem in machine learning. The objective here is to learn a mapping from the example space to an ordered set of the labels. The ordinal classifiers are routinely used in social sciences, information retrieval or computer vision. Ranking using ordinal regression is frequently used in settings where it is natural to rank or rate instances. For example, in online retail stores (e.g. Amazon, eBay etc.), product ratings can be generated using ordinal regression considering customer reviews as features. Detecting the age of a person from its face image, determining a users interest level in a movie using the user's past movie ratings etc are some other examples where ordinal regression is useful.

A ranking classifier is defined using a linear function and a set of $K-1$ thresholds ($K$ be the number of classes). Each threshold corresponds to a class. Thus, the thresholds should have the same order as their corresponding classes. The rank (class) of an observation is predicted based on the relative position of the linear function value with respect to different thresholds. Non-linear ranking classifiers can also be learnt by using an appropriate nonlinear feature transformation with the help of kernel methods. The discriminative methods for learning ranking classifier are based on minimizing the empirical risk with an appropriate regularization. Different batch learning algorithms for ordinal regression based on large margin have been discussed in \cite{Shashua:2002,Chu:2005,Herbrich99}. Batch algorithms use all the data simultaneously for learning the parameters. In the case of big data, it requires huge amount of computation time and memory to solve the optimization problem. In contrast, online learning updates its hypothesis based on a single example at every instant. Perceptron algorithm is extended for online learning of ranking classifiers \cite{Crammer:2001}. Harrington et. al \cite{Harrington:2003,DBLP:journals/corr/ChaudhuriT15} proposed online learning of large margin classifiers for ranking. Passive-aggressive (PA) \cite{Crammer:2006} is another principled method of learning classifiers in online fashion. The updates made by PA are more aggressive to make the loss incurred on the current example zero. This approach can be applied to learning multi-class classification, regression, multitask learning etc. A variant of passive-aggressive learning for multi-class classifier is proposed in \cite{Matsushima2010ExactPA}. PA algorithms for ranking have not been well addressed in the literature.

In all the above approaches, it is assumed that the training data contains exact labels for each observation. However, in many situations, we get interval labels instead of exact label \cite{pmlr-v39-antoniuk14}. For example, in case of predicting product ratings, we can get an entire interval of ratings (etc. 1-3, 4-7, 8-10) from different customers. Similarly, while learning a model for predicting human age, we can get a range of values in which the actual age of the person lies (e.g. 0-9, 10-19, 20-29, $\cdots$, 90-99). A large margin batch algorithm for learning to rank is proposed in \cite{pmlr-v39-antoniuk14} using interval labels.
 
In this paper, we propose passive-aggressive algorithms for ranking. These algorithms not only utilize the ordering of the class labels, but also are generic enough to accept both exact as well as interval labels in the training data. To the best of our knowledge this is the first work in that direction. 
Our key contributions in this paper are as follows. 
\begin{itemize}
\item[a] We derive update rules for PA, PA-I and PA-II. PA algorithms update the parameters at a trial $t$ by minimizing convex optimization problems. We find the exact solution of these optimization problems. We propose support class algorithm (SCA) which, at any trial, finds active constraints in the KKT optimality conditions to find the support class set. Support class set describes the thresholds that need to be updated in addition to the weight vector. We show that SCA correctly finds the support classes.
\item[b] We show that the proposed PA algorithms implicitly maintain the ordering of the thresholds after every trial. 
\item[c] We provide the mistake bounds for the proposed algorithms in both general and ideal cases.  
\item[d] We perform extensive simulations of the proposed algorithms on various datasets and show their effectiveness by comparing the results with different other algorithms.
\end{itemize}
This paper is organized as follows. In Section~2, we discuss a generic framework of learning to rank using interval (exact) labels. In Section~3, we derive the update rules for PA, PA-I and PA-II. The order preservation guarantees of proposed algorithms is discussed in Section~4. In Section~5, we discuss the mistake bounds. Experiments are presented in Section~6. We conclude our paper with some remarks in Section~7.

\section{Learning to Rank Using Interval (Exact) Labels}
Let $\X \subset \R^d$ be the instance space and $\Y = \{1,\ldots,K\}$ be the label space. For every instance $\xx \in \X$, an interval label $[y_l,y_r]\in \Y\times \Y$ is given. The exact (actual) label $y$ lie in the interval label. When $y_l=y_r$ for all the examples, it becomes the exact label scenario. 
Let $S = \{(\xx^1,y^1_l,y_r^1),\ldots,(\xx^T,y^T_l,y_r^T)\}$ be the training set. The goal here is to learn a ranking classifier using the training set $S$ which can predict accurate label for an unseen example. A ranking classifier consists of a function $f:\X \rightarrow \R$ and ordered thresholds $\theta_1\leq \theta_2 \leq \ldots \leq \theta_{K-1}$. Ranking classifier is defined as
\begin{align}
\label{eq:ranking-function}
h(\xx)=\min_{i\in [K]}\big{\{}i\;:\;f(\xx)-\theta_i <0 \big{\}}
\end{align}
where $\theta_K=\infty$ and $[K]=\{1,\ldots,K\}$. Let $f$ be a linear function of $\xx$, which means $f(\xx)=\ww.\xx$. We can use the kernel trick to generalize for non-linear functions.
Since we consider interval labels for each example, we use interval insensitive loss \cite{pmlr-v39-antoniuk14} to capture the discrepancy between the interval label and the predicted label.
\begin{align*}
L_{I}^{MAE}(f(\xx),\thetaa,y_l,y_r)=\sum_{i=1}^{y_l-1}\I_{\{f(\xx)< \theta_i\}}+\sum_{i=y_r}^{K-1}\I_{\{f(\xx)\geq \theta_i\}}
\end{align*}
Where subscript $I$ stands for interval and superscript $MAE$ stands for mean absolute error. This, loss function takes value $0$ whenever $\theta_{y_l} \leq f(\xx) \leq \theta_{y_r}$. However, this loss function is discontinuous.
A convex surrogate \cite{pmlr-v39-antoniuk14} of this loss function is as follows.
\begin{align}
\nonumber &L_{IMC}(f(\xx),\thetaa,y_l,y_r) = \sum_{i=1}^{y_l-1}l_i + \sum_{i=y_r}^{K-1}l_i\\
&=\sum_{i=1}^{y_l-1}[1-f(\xx)+\theta_i]_+ + \sum_{i=y_r}^{K-1}[1+f(\xx) - \theta_i]_+ \label{eq:loss}
\end{align}
where $\thetaa = [\theta_1\;\ldots\;\theta_{K-1}]$ and $[z]_+=\max(0,z)$. When $y_l=y_r$, then the loss above leads to the implicit threshold constraint formulation  described in \cite{Chu:2005}. 

\section{Exact Passive Aggressive Algorithms for Learning to Rank}
Passive-aggressive (PA) \cite{Crammer:2006} is a principled approach for supervised learning in online fashion. Here, we develop PA algorithms for ranking which can learn even when we have interval labels. The proposed approach is based on the interval insensitive loss described in Eq.~(\ref{eq:loss}). We derive the update equations for PA, PA-I, PA-II separately.

\subsection{PA Algorithm}
Let $\xx^t$ be the example being observed at trial $t$. Let $\ww^t\in \R^d$ and $\thetaa \in \R^{K-1}$ be the parameters of the ranking function at time $t$. We now use these parameters to predict the label. Then we observe the actual label(s). PA algorithm finds $\ww^{t+1}$ and $\thetaa^{t+1}$ which are closest to $\ww^{t}$ and $\thetaa^{t}$ such that the loss $L_{IMC}$ becomes zero for the current example. Thus, 
\begin{align}
\nonumber  \ww^{t+1},\thetaa^{t+1}  
& = \underset{\ww, \thetaa}{\arg\min} \;\; \frac{1}{2} \lVert \ww - \ww^{t} \rVert^{2} + \Vert \thetaa-\thetaa^t\Vert^{2}\\
& \;\;\;\;\;s.t. \begin{cases}
\mathbf{w}.\mathbf{x}^{t} - \theta_{i} \geq 1 &  i=1,\ldots, y_{l}^{t}-1 \\
 \mathbf{w}.\mathbf{x}^{t} - \theta_{i} \leq -1 & i=y_{r}^{t}, \cdots, K-1
\end{cases}
\end{align}
Lagrangian for the the above objective function is as follows.
\begin{align*}
 &\mathcal{L}(\mathbf{w}, \thetaa, \lambdaa, \muu) =  \frac{1}{2} \lVert \mathbf{w} - \mathbf{w}^{t} \rVert^{2}  + \frac{1}{2}\Vert\thetaa - \thetaa^t \Vert^{2}  \\
 &+ \sum_{i=1}^{y_{l}^{t}-1} \lambda_{i} \ ( 1 + \theta_{i} - \mathbf{w} \cdot \mathbf{x}^{t} ) + \sum_{i=y_{r}^{t}}^{K-1} \mu_{i} \ ( 1 + \mathbf{w} \cdot \mathbf{x}^{t} - \theta_{i} ) 
\end{align*}
where $\lambdaa=[\lambda_1\;\;\ldots\;\;\lambda_{y_l^t-1}]$, $\muu=[\mu_{y_r^t}\;\;\ldots\;\;\mu_{K-1}]$ such that $\lambda_{i} \geq 0\;i=1, \ldots, y_{l}^{t}-1$ and $\mu_{i} \geq 0,\;i=y_{r}^{t}, \ldots, K-1$. The KKT conditions of optimality are as follows.
\begin{align*}
& \mathbf{w} =  \mathbf{w}^{t}   +   (\sum_{i=1}^{y_{l}^{t}-1} \lambda_{i}^{t} -\sum_{i=y_{r}^{t}}^{K-1} \mu_{i}^{t})\mathbf{x}^{t} \\
& \theta_{i} = \theta_{i}^{t}  - \lambda_{i}^{t};\;\lambda_{i}\geq 0,\;i=1\ldots y_l^t-1 \\
& \mu_{i}\geq 0;\;\theta_{i} = \theta_{i}^{t}  + \mu_{i}^{t},\;i=y_r^t\ldots K-1\\
& 1 + \theta_{i} - \mathbf{w} \cdot \mathbf{x}^{t}\leq0;\; \lambda_{i} \ ( 1 + \theta_{i} - \mathbf{w} \cdot \mathbf{x}^{t} )=0,\;\forall i=1\ldots y_l^t-1\\
& 1 + \mathbf{w} \cdot \mathbf{x}^{t} - \theta_{i} \leq0;\; \mu_{i} ( 1 + \mathbf{w} \cdot \mathbf{x}^{t} - \theta_{i} )=0,\; \forall i=y_r^t\ldots K-1
\end{align*}
Let $S_l^t = \{1\leq i \leq y_l^t-1|\lambda_i^t>0\}$ be the left support set. Similarly, let $S_r^t=\{y_r^t \leq i\leq K-1|\mu_i>0\}$ be the right support set. Thus, optimal $\ww$ can be rewritten as 
$\ww = \ww^t + (\sum_{i \in S_l^t}\lambda_i - \sum_{i \in S_r^t}\mu_i)\xx^t=a^t\xx^t$ where $a^t = \sum_{i \in S_{l}^{t}} \lambda_{i}^{t}  -  \sum_{i \in S_{r}^{t}} \mu_{i}^{t}$. Also,
\begin{align}
\label{eq:PA-SC}
\mathbf{w} \cdot \mathbf{x}^{t}- \theta_{i} =\begin{cases}
1, &\forall i \in S_l^t\\
-1,&\forall i \in S_r^t
\end{cases}
\end{align}
 Using optimal $\ww$ in Eq.~(\ref{eq:PA-SC}), we get
\begin{align*}
\lambda_i &= 1-\ww^t.\xx^t-\theta_i^t -a^t\Vert \xx^t\Vert^2= l_i^t - a^t\Vert \xx^t\Vert^2,\;\forall i \in S_l^t\\
\mu_i &= 1-\theta_i^t+\ww^t.\xx^t +a^t\Vert \xx^t\Vert^2=l_i^t + a^t\Vert \xx^t\Vert^2,\;\forall i \in S_r^t
\end{align*}
But, $a^t  = \sum_{i \in S_{l}^{t}}(l_i^t-a^t\Vert \xx^t\Vert^2)-\sum_{i \in S_{r}^{t}}(l_i^t+a^t\Vert \xx^t\Vert^2)=\sum_{i \in S_{l}^{t}}l_i^t-\sum_{i \in S_{r}^{t}}l_i^t - a^t(\vert S_l^t\vert + \vert S_r^t \vert)\Vert \xx^t\Vert^2$. Which means, $a^t = \frac{\sum_{i \in S_{l}^{t}}l_i^t-\sum_{i \in S_{r}^{t}}l_i^t}{1+(\vert S_l^t\vert + \vert S_r^t \vert)\Vert \xx^t\Vert^2}$.
The complete description of the PA algorithm is as given in Algorithm~\ref{SPOR}. Note that PA updates assume that at every trial $t$, sets $S_l^t$ and $S_r^t$ are known. We will now discuss the procedure for determining the support sets $S_l^t$ and $S_r^t$.  
\begin{algorithm}[ht]
\caption{PA Algorithm}\label{SPOR}
\begin{algorithmic}[0]
\State {\bf Input} Training set $S$
\State {\bf Initialize} $\mathbf{w}^{0}$
and $ \thetaa^{0}$
\For{$t=1,\cdots,T$}
\State $\mathbf{x}^{t} \gets$ \textit{randomly sample an instance from $S$}
\State Predict: $\hat{y}^{t} = \mathbf{w}^{t}.\mathbf{x}^{t} $
\State Observe $ y_{l}^{t},y_{r}^{t}$  
\State $ l_i^t = \max(0, 1 + \theta_{i}^{t} -\mathbf{w}^{t}.\mathbf{x}^{t}),\; i=1\ldots y_{l}^{t}-1$ 
\State $l_i^t = \max(0,1 + \mathbf{w}^{t}.\mathbf{x}^{t}-\theta_{i}^{t}),\; i=y_{r}^{t}\ldots K-1 $
\State $S_l^t,S_r^t =$ SCA($l_1^t,\ldots,l_{y_l^t-1}^t,l_{y_r^t}^t,\ldots, l_{K-1}^t,y_l^t,y_r^t\mathbf{x}^{t}$)
\State Update: 
\begin{align*}
\mathbf{w} &= \mathbf{w}^{t} + a^t\mathbf{x}^{t} \\
\theta^{t+1}_i &= \theta_i^t - l_i^t + \lVert \mathbf{x}^{t} \rVert^{2} a^{t},\;\;\forall i \in S_l^{t}\\
\theta_i^{t+1} &= \theta_i^t + l_i^t + \lVert \mathbf{x}^{t} \rVert^{2} a^{t},\;\;\forall i \in S_r^{t}
\end{align*}
\EndFor
\end{algorithmic}
\end{algorithm} 

\subsubsection{Determining Support Sets $S_l^t$ and $S_r^t$} 
Note that the loss decreases as we move away from the correct label range on either side. We initialize with $ S_{l}^t = \{y_l^t-1\} $ and $ S_{r}^t = \{y_r^t\} $. We can easily verify that with this initialization $\lambda_{y_l^t-1}^t,\mu_{y_r^t}^t >0$. We start with considering the the threshold $\theta_{y_l^t-2}^t$ and find corresponding Lagrange multiplier value $\lambda_{y_l^t-2}^t$. If it appears positive, then we add it to the support set $S_l^t$, else consider threshold $\theta_{y_r^t+1}^t$. We check if $\mu_{y_r^t+1}^t$ is positive. If so, we add it to $S_r^t$. We repeatedly check this for all the thresholds.
The detailed approach for constructing support sets is described in Algorithm~\ref{SCA}.
\begin{algorithm}[h]
\caption{Support Class Algorithm (SCA)}\label{SCA}
\begin{algorithmic}[0]
\State \textbf{Input:} $ y_{l}^{t} $, $ y_{r}^{t} $and $l_{i}^{t},\;i=1\ldots,K-1$ 
\State \;\;\;\;\;\;\;\;\; $\thetaa^t = \{ \theta_{1}^{t}, \cdots, \theta_{K-1}^{t} \} $
\State {\bf Initialize: }$ S_{l}^t=\{y_l^t-1\}$, $S_{r}^t = \{y_r^t\}$, flag = 1, p = $ y_{l}^{t}-2 $, q = $ y_{r}^{t}+1 $
\While{flag = 1}
\If{ $p>0$ }
\If{$l_{p}^{t} - \frac{\Vert\mathbf{x}^{t} \Vert^{2} ( l_{p}^{t} + \sum_{j \in S_{l}^t} l_{j}^{t} - \sum_{j \in S_{r}^t} l_{j}^{t})}{1 + \Vert\mathbf{x}^{t} \Vert^{2} ( 1 + \vert S_{l}^t \vert + \vert S_{r}^t \vert ) }  >  0 $}
\State $ S_{l}^t = S_{l}^t \cup \{ p \} $
\State $p=p-1$
\State flag = 1
\Else
\State flag=0
\EndIf
\EndIf
\If{$q<K$}
\If{$l_{q}^{t} + \frac{\Vert\mathbf{x}^{t} \Vert^{2} (\sum_{j \in S_{l}^t} l_{j}^{t} - l_{q}^{t} - \sum_{j \in S_{r}^t} l_{j}^{t})}{1 + \Vert\mathbf{x}^{t} \Vert^{2} ( 1 + \vert S_{r}^t \vert + \vert S_{l}^t \vert ) } > 0 $}
\State $ S_{r}^t = S_{r}^t \cup \{q\} $
\State $q=q+1$
\State flag = 1
\Else
\State flag=0
\EndIf
\EndIf
\EndWhile
\end{algorithmic}
\end{algorithm} 
Following Lemma shows the correctness of the SCA algorithm discussed.
\begin{lemma}
\label{thm1}
 Assume that $S_{l}^t \neq \phi$. Let, $k \notin S_{l}^t$ and $k+1 \in S_{l}^t$. Then, $k' \notin S_l^t,\;\forall k'< k$.
\end{lemma}
\begin{proof} We are given that $ k \notin S_{l}^t $. Thus,
\begin{align*}
& \lambda_k^t = l_{k}^{t} - \frac{  \Vert \mathbf{x}^{t} \Vert^{2} (l_{k}^{t}  + \sum_{j \in S_{l}^t} l_{j}^{t} - \sum_{j \in S_{r}^t} l_{j}^{t}) }{1 +  \Vert \mathbf{x}^{t} \Vert^{2} ( \vert S_{l}^t \vert + 1 + \vert S_{r}^t \vert ) } \leq 0 
\end{align*}
$\forall k'<k$, we know that $l_{k'}^t \leq l_k^t$. Now, if we try to add $k'$ in $S_l^t$, then
\begin{align*}
 &\lambda_{k'}^t  = l_{k'}^{t} - \frac{  \Vert \mathbf{x}^{t} \Vert^{2} (l_{k'}^{t}  + \sum_{j \in S_{l}^{t}} l_{j}^{t} - \sum_{j \in S_{r}^{t}} l_{j}^{t}) }{1 +  \Vert \mathbf{x}^{t} \Vert^{2} ( 1+ \vert S_{l}^{t} \vert + \vert S_{r}^{t} \vert ) }\\
 &= \frac{l_{k'}^{t}\left(1 +  \Vert \mathbf{x}^{t} \Vert^{2} ( \vert S_{l}^{t} \vert + \vert S_{r}^{t} \vert )\right)}{1 +  \Vert \mathbf{x}^{t} \Vert^{2} ( 1+ \vert S_{l}^{t} \vert + \vert S_{r}^{t} \vert ) } - \frac{  \Vert \mathbf{x}^{t} \Vert^{2} (\sum_{j \in S_{l}^{t}} l_{j}^{t} - \sum_{j \in S_{r}^{t}} l_{j}^{t}) }{1 +  \Vert \mathbf{x}^{t} \Vert^{2} ( 1+ \vert S_{l}^{t} \vert + \vert S_{r}^{t} \vert ) }\\
 &\leq\frac{l_{k}^{t}\left(1 +  \Vert \mathbf{x}^{t} \Vert^{2} ( \vert S_{l}^{t} \vert + \vert S_{r}^{t} \vert )\right)}{1 +  \Vert \mathbf{x}^{t} \Vert^{2} ( 1+ \vert S_{l}^{t} \vert + \vert S_{r}^{t} \vert ) } - \frac{  \Vert \mathbf{x}^{t} \Vert^{2} (\sum_{j \in S_{l}^{t}} l_{j}^{t} - \sum_{j \in S_{r}^{t}} l_{j}^{t}) }{1 +  \Vert \mathbf{x}^{t} \Vert^{2} ( 1+ \vert S_{l}^{t} \vert + \vert S_{r}^{t} \vert ) }\\
 & =l_{k}^{t} - \frac{  \Vert \mathbf{x}^{t} \Vert^{2} (l_{k}^{t}  + \sum_{j \in S_{l}^{t}} l_{j}^{t} - \sum_{j \in S_{r}^{t}} l_{j}^{t}) }{1 +  \Vert \mathbf{x}^{t} \Vert^{2} ( 1+ \vert S_{l}^{t} \vert + \vert S_{r}^{t} \vert ) }=\lambda_k \leq 0
\end{align*}
Thus, $k' \notin S_l^t$.

\end{proof}
Thus, if a threshold doesn't belong to the left support class $S_{l}^t$ then all the threshold on its left side also don't belong to $ S_{l}^{t} $. Hence, if we start adding the classes in the support class set in decreasing order of respective losses, then this would ensure that we end up with only those classes which have positive Lagrange multiplier. Similarly, it can be shown that if $k-1\in S_{r}^t$ and $k \notin S_r^t$, then $k' \notin S_r^t,\; \forall k'>k$. Which means, if a threshold doesn't belong the right support class $ S_{r}^t $ then all the threshold on its right side also don't belong to $ S_{r}^t $. 

\subsection{PA-I}
The PA-I find the new parameters by minimizing the following objective.
\begin{align*}
\nonumber & \underset{\ww, \thetaa}{\arg\min} \;\; \frac{1}{2} \lVert \ww - \ww^{t} \rVert^{2}  +  \frac{1}{2}\Vert \thetaa - \thetaa^t \Vert^2+ C\left(\sum_{i=1}^{y_l^t-1}\xi_i+\sum_{y_r^t}^{K-1}\xi_i\right)\\
& \;\;\;\;\;s.t. \begin{cases}
\mathbf{w}.\mathbf{x}^{t} - \theta_{i} \geq 1 - \xi_{i} &  i=1,\ldots, y_{l}^{t}-1 \\
 \mathbf{w}.\mathbf{x}^{t} - \theta_{i} \leq -1 + \xi_{i} & i=y_{r}^{t}, \ldots, K-1 \\
 \xi_{i} \geq 0 & i=1,\ldots y_l^t-1, y_r^t\ldots K-1
\end{cases}
\end{align*}
where $C$ is the aggressiveness parameter. We skip the derivation of PA-I updates as it follows the same steps used in case of PA.
PA-I updates the parameters as follows.
\begin{align*}
\ww &=\ww^t+(\sum_{i\in S_l^t}\lambda_i - \sum_{i\in S_r^t}\mu_i)\xx^t\\
\lambda_i &= \min(C,l_i^t - a^t\Vert \xx^t \Vert^2),\;i\in S_l^t\\
\mu_i&=\min(C,l_i^t + a^t\Vert \xx^t \Vert^2),\;i\in S_r^t
\end{align*}
where $S_l^t = \{1\leq i\leq y_l^t-1\;|\;\lambda_i>0\}$, $S_r^t = \{y_r^t \leq i \leq K-1\;|\;\mu_i>0\}$ and $a^t= \sum_{i \in S_{l}^{t}} \lambda_{i}^{t} - \sum_{i \in S_{r}^{t}} \mu_{i}^{t}$. PA-I uses the same steps as described in Algorithm~\ref{SPOR} except that it uses a different approach to determine the support sets $S_l^t$ and $S_r^t$. We use an iterative approach to find the support sets.  We first find the values of all the $\lambda_i^t$ and $\mu_i^t$ and then compute $a^t$. We repeat it till all the values get converge. Then we include an $i$ in $S_l^t$ or $S_r^t$ based on whether $\lambda_i>0$ or $\mu_i>0$. Support class algorithm (SCA-I) for PA-I is discussed in Algorithm~\ref{SCA-I}.

\begin{algorithm}[h]
\caption{Support Class Algorithm-I (SCA-I)}\label{SCA-I}
\begin{algorithmic}[0]
\State \textbf{Input:} $ y_{l}^{t} $, $ y_{r}^{t} $, $ \ww^t.\xx^t $ and $l_{i}^{t},\;i\in[K-1]$ 
\State \;\;\;\;\;\;\;\;\; $\Theta^t = \{ \theta_{1}^{t}, \cdots, \theta_{K-1}^{t} \} $
\State {\bf Initialize: }$ S_{l}^t=\{y_l^t-1\}$, $S_{r}^t = \{y_r^t\} $, $p = y_{l}^{t}-2 $,
\State \hspace{0.6in}$q =  y_{r}^{t}+1 $
\While{$ \lambda_{i}^1,\ldots,\lambda_{y_l^t-1}^t,\mu_{y_r^t}^t,\ldots,\mu_{K-1}^t$ do not converge}
\For{ $ i=\text{p}, \cdots, 1$ }
\If{$\min(C, l_{i}^{t} - a^{t}\Vert\mathbf{x}_{t}\Vert^{2}) > 0 $}
\State $ S_{l}^t = S_{l}^t \cup \{ i \}$
\Else
\If{$ i \in S_{l}^{t} $}
\State $ S_{l}^{t} = S_{l}^{t} - \{ i \} \ ; \ \lambda_{i}^{t} = 0  $
\EndIf
\EndIf
\EndFor
\For{ $ i=\text{q}, \cdots, K-1$ }
\If{$\min(C, l_{i}^{t} + a^{t}\Vert\mathbf{x}_{t}\Vert^{2}) > 0 $}
\State $ S_{r}^t = S_{r}^t \cup \{ i \}$
\Else
\If{$ i \in S_{r}^{t} $}
\State $ S_{r}^{t} = S_{r}^{t} - \{ i \} \ ; \ \mu_{i}^{t} = 0  $
\EndIf
\EndIf
\EndFor
\EndWhile
\end{algorithmic}
\end{algorithm}

\subsection{PA-II}
PA-II finds the new parameters by minimizing the following objective function.
\begin{align}
\nonumber  \ww^{t+1},\thetaa^{t+1} & = \underset{\ww, \thetaa}{\arg\min} \;\; \frac{1}{2} \lVert \ww - \ww^{t} \rVert^{2} \ + \ \frac{1}{2}\Vert \thetaa-\thetaa^t \Vert^2  \\ \nonumber &\;\;\;\;+C\left(\sum_{i=1}^{y_l^t-1} \xi_{i}^{2} + \sum_{i=y_r^t}^{K-1}\xi_{i}^2\right) \\
& \;\;\;\;\;s.t. \begin{cases}
\mathbf{w}.\mathbf{x}^{t} - \theta_{i} \geq 1 - \xi_{i} &  i=1,\ldots, y_{l}^{t}-1 \\
 \mathbf{w}.\mathbf{x}^{t} - \theta_{i} \leq -1 + \xi_{i} & i=y_{r}^{t}, \ldots, K-1
\end{cases}
\end{align}
The PA-II update equations are as follows. 
\begin{align*}
\mathbf{w}^{t+1} &= \mathbf{w}^{t} + a^{t}\mathbf{x}^{t} \\
\theta_i^{t+1} &=\theta_i^t -\lambda_i^t,\;\forall i \in S_l^t\\
\theta_i^{t+1}&=\theta_i^t +\mu_i^t,\;\forall i \in S_r^t
\end{align*}
where $\lambda_{i}^{t} = \frac{l_{i}^{t} - a^{t}\Vert\mathbf{x}^{t}\Vert^{2}}{1 + \frac{1}{2C}}$, $\mu_{i}^{t} =  \frac{ l_{i}^{t} + a^{t}\Vert\mathbf{x}_{t}\Vert^{2}}{1 + \frac{1}{2C}}$ and 
$a^{t} = \frac{\sum_{i \in S_{l}^{t}} l_{i}^{t}  -  \sum_{i \in S_{r}^{t}} l_{i}^{t}}{1 \ + \frac{1}{2C}  +  \lVert \mathbf{x}^{t} \rVert^{2} \{ \lvert S_{l}^{t} \rvert + \lvert S_{r}^{t} \rvert \}} $. The support sets $S_l^t$ and $S_r^t$ can be found in the similar way as in SCA.

\section{Correctness of PA Algorithms}

Now, we will show that our approach inherently maintains the ordering  of thresholds in each iteration.

\begin{theorem}(Order preservation of thresholds using PA algorithm)
Let $\theta_1^t\leq \ldots\leq \theta_{K-1}^t$ be the thresholds at trial $t$. Let $\theta_1^{t+1},\ldots,\theta_{K-1}^t$ be the updated thresholds using PA. Then, $\theta_1^{t+1}\leq \ldots \leq \theta_{K-1}^t$.
\end{theorem}
\begin{proof}
We need to analyse following different cases.
\begin{enumerate}
\item  We know that $\theta_k^{t+1}=\theta_k^t,\;k=y_l^t\ldots y_r^t-1$. Thus, $\theta_{y_l^t}^{t+1} \leq\ldots \leq \theta_{y_r^t-1}^{t+1}$.
\item $\forall k \in S_{l}^t $, we see that
\begin{align*}
\theta_k^{t+1} 
&=-1+\ww.\xx + \frac{\Vert \mathbf{x}^{t} \Vert^{2}(\sum_{i \in S_{l}^t} l_{i}^{t} - \sum_{i \in S_{r}^t} l_{i}^{t})}{1 + \Vert \mathbf{x}^{t} \Vert^{2} ( \vert S_{l}^t \vert + \vert S_{r}^t \vert ) }
\end{align*}
Thus, all the thresholds in the set $S_l^t$ are mapped to the same value and hence the ordering is preserved.
\item $\forall k \in S_{r}^t $, we see that
\begin{align*}
\theta_k^{t+1} 
&=1+\ww.\xx + \frac{\Vert \mathbf{x}^{t} \Vert^{2}(\sum_{i \in S_{l}^t} l_{i}^{t} - \sum_{i \in S_{r}^t} l_{i}^{t})}{1 + \Vert \mathbf{x}^{t} \Vert^{2} ( \vert S_{l}^t \vert + \vert S_{r}^t \vert ) }
\end{align*}
All the thresholds in the set $S_r^t$ are mapped to the same value and hence the ordering is preserved.
\item Let $k,k+1 \in [y_l^t-1] \triangle S_{l}^t $ where $\triangle$ is symmetric difference between two sets. Then $\theta_{k+1}^{t+1}-\theta_k^{t+1} = \theta_{k+1}^t-\theta_k^t\geq 0$.
\item Let $k \in [y_l^t-1] \triangle S_{l}^t $ and $ k+1 \in S_{l}^t$. Then, using Theorem~\ref{thm1}, we get
\begin{align}
 \nonumber  l_{k}^{t} &\leq  \frac{\Vert \mathbf{x}^{t} \Vert^{2}(l_{k}^{t} + \sum_{i \in S_{l}^t} l_{i}^{t} - \sum_{i \in S_{r}^t} l_{i}^{t})}{1 + \Vert \mathbf{x}^{t} \Vert^{2} ( \vert S_{l}^t \vert + 1 + \vert S_{r}^t \vert ) }\\
 &\leq \frac{\Vert \mathbf{x}^{t} \Vert^{2} (\sum_{i \in S_{l}^t} l_{i}^{t} - \sum_{i \in S_{r}^t} l_{i}^{t})}{1 + \Vert \mathbf{x}^{t} \Vert^{2} ( \vert S_{l}^t \vert +\vert S_{r}^t \vert ) } =a^{t}\Vert \mathbf{x}^{t} \Vert^{2}
\label{eq:bound}
\end{align}
Then, using (\ref{eq:bound}), $\theta_{k+1}^{t+1}-\theta_k^{t+1} = \theta_{k+1}^t - l_{k+1}^t + a^{t}\Vert \mathbf{x}^{t} \Vert^{2}-\theta^t_k = \theta_{k+1}^t -(l_k^t - \theta_k^t + \theta_{k+1}^t)+ a^{t}\Vert \mathbf{x}^{t} \Vert^{2}-\theta^t_k = -l_k^t +a^{t}\Vert \mathbf{x}^{t} \Vert^{2} \geq 0$.
 \item Let $k,k+1 \in \{y_r^t,\ldots,K-1\} \triangle S_r^t $, then $\theta_{k+1}^{t+1}-\theta_k^{t+1} = \theta_{k+1}^t-\theta_k^t\geq 0$.  
 \item Let $k+1 \in \{y_r^t,\ldots,K-1\} \triangle S_r^t $ and $ k \in S_{r}^t$. Then, 
\begin{align}
 \nonumber l_{k+1}^{t} &\leq - \frac{\Vert \mathbf{x}^{t} \Vert^{2}(\sum_{i \in S_{l}^t} l_{i}^{t} - \sum_{i \in S_{r}^t} l_{i}^{t}-l_{k+1}^t)}{1 + \Vert \mathbf{x}^{t} \Vert^{2} ( \vert S_{l}^t \vert + 1 + \vert S_{r}^t \vert ) }\\
 &\leq -\frac{\Vert \mathbf{x}^{t} \Vert^{2} (\sum_{i \in S_{l}^t} l_{i}^{t} - \sum_{i \in S_{r}^t} l_{i}^{t})}{1 + \Vert \mathbf{x}^{t} \Vert^{2} ( \vert S_{l}^t \vert +\vert S_{r}^t \vert ) } =-a^{t}\Vert \mathbf{x}^{t} \Vert^{2}
\label{eq:bound1}
\end{align}
Then, using (\ref{eq:bound1}), $\theta_{k+1}^{t+1}-\theta_k^{t+1} = \theta_{k+1}^t - \theta^t_k - l_{k}^t - a^{t}\Vert \mathbf{x}^{t} \Vert^{2} = \theta_{k+1}^t -(l_{k+1}^t - \theta_k^t + \theta_{k+1}^t)- a^{2}\Vert \mathbf{x}^{t} \Vert^{2}-\theta^t_k = -l_{k+1}^t -a^{t}\Vert \mathbf{x}^{t} \Vert^{2} \geq 0$. 
\end{enumerate} 
This completes the proof.
\end{proof}

\begin{theorem}(Order preservation of thresholds using PA-I)
Let $\theta_1^t\leq \ldots\leq \theta_{K-1}^t$ be the thresholds at trial $t$. Let $\theta_1^{t+1},\ldots,\theta_{K-1}^{t+1}$ be the updated thresholds using PA-I. Then, $\theta_1^{t+1}\leq \ldots \leq \theta_{K-1}^t$.
\end{theorem}
\begin{proof}
The proof follows in the same manner as PA algorithm. We only consider here following two cases.
\begin{enumerate}
\item $k+1 \in S_l^t$ and $k \in [y_l^t-1]\triangle S_l^t$.  Thus, $\lambda_{k}^{t} < 0$. Which means, $l_{k}^{t} - a^{t}\Vert \mathbf{x}_{t} \Vert^{2} < 0$ as $C>0$. Also, $\lambda_{k+1}^{t} = \min(C, l_{k+1}^{t} - a^{t}\Vert\mathbf{x}^{t}\Vert^{2})  > 0$. When $ \lambda_{k+1}^{t} = l_{k+1}^{t} - a^{t}\vert \mathbf{x}^{t}\Vert^{2} $, we see that \begin{align*}
\theta_{k+1}^{t+1}&-\theta_k^{t+1} =\theta_{k+1}^t-l_{k+1}^{t} + a^{t}\vert \mathbf{x}^{t}\Vert^{2}-\theta_k^t\\
&=\theta_{k+1}-(l_k^t-\theta_k^t+\theta_{k+1}^t)+a^{t}\vert \mathbf{x}^{t}\Vert^{2} -\theta_k^t\\
&=-l_k^t + a^{t}\vert \mathbf{x}^{t}\Vert^{2}\geq 0
\end{align*}
When $\lambda_{k+1}^{t} = C $ ($C\leq l_{k+1}^t-a^t\Vert \mathbf{x}_{t}\Vert^{2}$), we have $\theta_{k+1}^{t+1}-\theta_k^{t+1} = \theta_{k+1}^t-C-\theta_k^t \geq \theta_{k+1}^t-l_{k+1}^{t} + a^{t}\Vert \mathbf{x}^{t}\Vert^{2}-\theta_k^t\geq 0$.
\item Let $ k,k+1\in S_{l}^t $. Thus, $\theta_{k+1}^{t+1}-\theta_k^{t+1} = \theta_{k+1}^t-\theta_k^t - \lambda_{k+1}^t + \lambda_k^t$. There can be four different cases as below.
\begin{enumerate}
\item When $\lambda_{k+1}^t=\lambda_k^t=C$. Thus, $\theta_{k+1}^{t+1}-\theta_k^{t+1}=\theta_{k+1}^{t}-\theta_k^{t}\geq 0$. Similar, is the case when $\lambda_k^t=C$, then $\lambda_{k+1}^t= C$ due to the fact that $l_{k+1}^{t} \geq l_{k}^{t}$.
\item Let $\lambda_k^t=l_{k}^{t} - a^{t}\Vert \mathbf{x}^{t} \Vert^{2}$ and $\lambda_{k+1}^t= l_{k+1}^{t} - a^{t}\Vert \mathbf{x}^{t} \Vert^{2}$. Thus, $\theta_{k+1}^{t+1}=\theta_k^{t+1} = -1+\ww^t.\xx^t + a^t\Vert \mathbf{x}_{t} \Vert^{2}$. 
\item Let $ \lambda_k^t=l_{k}^{t} - a^{t}\Vert \mathbf{x}^{t} \Vert^{2}$ and $\lambda_{k+1}^t=C$. We see that $\theta_k^{t+1} = -1+\ww^t.\xx^t + a^t\Vert \mathbf{x}^{t} \Vert^{2}$ and $\theta_{k+1}^{t+1}=\theta_{k+1}^{t}-C \geq \theta_{k+1}^t-l_{k+1}^t + a^t\Vert \mathbf{x}^{t} \Vert^{2} = -1+\ww^t.\xx^t + a^t\Vert \mathbf{x}^{t} \Vert^{2}$. Thus, $\theta_{k+1}^{t+1}-\theta_k^{t+1} \geq 0$.
\end{enumerate}
 \end{enumerate}
Similar arguments can be given for the right support class $S_{r}^{t}$ and hence, we skip the proof for it.
\end{proof}

\begin{theorem}(Order preservation of thresholds using PA-II) Let $\theta_1^t\leq \ldots\leq \theta_{K-1}^t$ be the thresholds at trial $t$. Let $\theta_1^{t+1},\ldots,\theta_{K-1}^t$ be the updated thresholds using PA-II. Then, $\theta_1^{t+1}\leq \ldots \leq \theta_{K-1}^t$.
\end{theorem}
The order preservation proof for PA-II works in the similar way as PA algorithm. 

\section{Mistake Bound Analysis}
We find the mistake bounds for the proposed PA algorithms under both general and ideal cases. 
In the ideal case, there exists a ranking function such that for every example, the predicted label lies in the label interval with certain margin guarantees. Thus, for every example, the loss incurred using it would be zero. In the general case, there does not exists an ideal classifier. Let $l_{i}^{t} $ be the loss due to $i^{th}$ threshold in trial $t$. Let $ l_{i}^{t*} $ denote the loss suffered due to $i^{th}$ threshold by the fixed predictor at trial $t$. We define $\Delta_t$ as follows.
\begin{align*}
\Delta_t &= \lVert \mathbf{w}^{t}  -  \mathbf{u} \rVert^{2}  -  \lVert \mathbf{w}^{t+1}  -  \mathbf{u} \rVert^{2}  + \Vert\thetaa^{t}  -  {\bf b}\Vert^2 - \Vert\thetaa^{t+1}  -  {\bf b}\Vert^{2} 
\end{align*}
Using the fact that $\ww^0 =\mathbf{0}$ and $\thetaa^0=\mathbf{0}$, we get
\begin{align}
\nonumber \sum_{i=1}^{T} \Delta_t &= \Vert \mathbf{w}^{0} - \mathbf{u} \Vert^{2} - \Vert \mathbf{w}^{T+1} - \mathbf{u} \Vert^{2} +  \Vert \thetaa^{0} - {\bf b}\Vert^{2}\\
&- \Vert\thetaa^{T+1} - {\bf b}\Vert^{2} \leq \Vert \mathbf{u} \Vert^{2} + \Vert \mathbf{b} \Vert^2
\label{eq:ub}
\end{align}
This gives an upper bound on the sum of $\Delta_t$. We see that $\theta_i^{t+1}=\theta_i^t,\;\forall i \notin S_l^t\cup S_r^t$. Thus,
\begin{align}
\nonumber &\Delta_t 
= -(a^t)^{2}\Vert \mathbf{x}^{t} \Vert^{2} - 2a^t\mathbf{x}^{t}.(\mathbf{w}^{t} - \mathbf{u}) - \sum_{i \in S_{l}^{t}}(\lambda_i^t)^{2} - \sum_{i \in S_{r}^{t}}( \mu_i^t)^{2}\\
&\;\;\;\;+ \sum_{i \in S_{l}^{t}} 2\lambda_i^t(\theta_i^t -b_{i}) - \sum_{i \in S_{r}^{t}} 2\mu_i^t(\theta_i^t - b_{i}) \nonumber
\end{align}
Note that $\theta_i^t = \ww^t.\xx^t+l_i^t-1,\;\forall i \in S_l^t$ and $\theta_i^t=1+\ww^t.\xx^t-l_i^t,\;\forall i \in S_r^t$.
Also, note that
$-b_{i}\geq 1 - \mathbf{u}.\mathbf{x}^{t} -l_{i}^{t*}, \forall i \in S_{l}^{t}$ and $b_{i}\geq 1 + \mathbf{u}.\mathbf{x}^{t} - l_{i}^{t*},\forall i \in S_{r}^{t}$. Thus,
\begin{align}
\nonumber \Delta_t &\geq -(a^t)^{2}\Vert \mathbf{x}^{t} \Vert^{2} - \sum_{i \in S_{l}^{t}}(\lambda_i^t)^{2} - \sum_{i \in S_{r}^{t}}( \mu_i^t)^{2} + \sum_{i \in S_{l}^{t}} 2\lambda_i^t(l_{i}^{t} -l_{i}^{t*}) \\
& + \sum_{i \in S_{r}^{t}} 2\mu_i^t(l_{i}^{t} - l_i^{t*}) \label{eq:delta-t}
\end{align}

Now, we find the mistake bound of the PA algorithm described in Algorithm~\ref{SPOR} in general case.
\begin{theorem}(Mistake Bound of PA in General Case)
\label{thm:regPA}
Let $(\mathbf{x}^{1}, y_{l}^{1}, y_{r}^{1}), \cdots,(\mathbf{x}^{T}, y_{l}^{T}, y_{r}^{T})$ be the sequence of examples.
 Let $c=\min_{t\in[T]}(y_r^t-y_l^t)$ and $R^2=\max_{t\in[T]}\Vert \xx^t\Vert^2$. Let $ \mathbf{v}=[\mathbf{u}'\;\;\mathbf{b}']'$ be the parameters of an arbitrary predictor ($\mathbf{u}\in \R^d $ and $ \mathbf{b} \in \R^{K-1}$). Then, the mistake bound of PA algorithm is given as
\begin{align*}
\sum_{t=1}^T\sum_{i=1}^{K-1}(l_i^t)^2 \leq D^2\left(\Vert\mathbf{v}\Vert+ 4(K-c-1) \sqrt{\sum_{t=1}^T\sum_{i=1}^{K-1}(l_i^{t*})^2}\right)^2
\end{align*}
where $D=\left(1+R^2(K-c-1)\right)$ and $R^2=\max_{t\in[T]}\Vert \xx^t \Vert^2$.
\end{theorem}
\begin{proof}
Using PA updates and Eq.~(\ref{eq:delta-t}), we get
\begin{align*}
 &\Delta_t 
 = -a_t^{2}\Vert\mathbf{x}^{t}\Vert^{2}\left[1+ \Vert\mathbf{x}^{t}\Vert^{2}(\vert S_l^t \vert + \vert S_r^t \vert)\right]+ \sum_{i \in S_{l}^{t}\cup S_r^t} (l_{i}^{t})^{2} \\
&\;\;\;\;+ \sum_{i \in S_{l}^{t}} 2(a^t\Vert\mathbf{x}^{t}\Vert^{2} - l_{i}^{t}) l_{i}^{t*}   -\sum_{i \in S_{r}^{t}} 2(l_{i}^{t} +a^{t}\Vert\mathbf{x}^{t}\Vert^{2}) l_{i}^{t*}\\
&\geq \frac{-(\sum_{i \in S_{l}^{t}} l_{i}^{t}  -  \sum_{i \in S_{r}^{t}} l_{i}^{t})^{2} \Vert \mathbf{x}^{t} \Vert^{2}}{1  + \Vert \mathbf{x}^{t} \Vert^{2} \{ \vert S_{l}^{t} \vert + \vert S_{r}^{t} \vert \}} + \sum_{i \in S_{l}^{t}\cup S_{r}^{t}} l_i^t[l_{i}^{t} -2l_i^{t*}]  \\
&\;\;\;\;-\frac{2\Vert \xx^t\Vert^2\left(\sum_{i\in S_l^t\cup S_r^t}l_i^t\sum_{j\in S_l^t\cup S_r^t}l_j^{t*}\right)}{1  + \Vert \mathbf{x}^{t} \Vert^{2} \{ \vert S_{l}^{t} \vert + \vert S_{r}^{t} \vert \}}\\
& \geq   - \frac{2(1  + \Vert \mathbf{x}^{t} \Vert^{2} \{ \vert S_{l}^{t} \vert + \vert S_{r}^{t} \vert +1\})\sum_{i\in S_l^t\cup S_r^t}l_i^t\sum_{j\in S_l^t\cup S_r^t}l_j^{t*}}{1  + \Vert \mathbf{x}^{t} \Vert^{2} \{ \vert S_{l}^{t} \vert + \vert S_{r}^{t} \vert \}}\\
&\;\;\;\;+\frac{\sum_{i\in S_l^t\cup S_r^t}(l_i^t)^2 }{1  + R^{2} (K-c-1)}\\
& \geq  \frac{\sum_{i\in S_l^t\cup S_r^t}(l_i^t)^2 }{D}  - 4\sum_{i\in S_l^t\cup S_r^t}l_i^t\sum_{j\in S_l^t\cup S_r^t}l_j^{t*}
\end{align*}
where $D=1  + R^{2} (K-c-1)$. We have used the fact that $\sum_{i\in S_l^t\cup S_r^t}l_i^tl_i^{t*}\leq \left(\sum_{i\in S_l^t\cup S_r^t}l_i^t\sum_{j\in S_l^t\cup S_r^t}l_j^{t*}\right)$ and $(\sum_{i\in S_l^t}l_i^t - \sum_{i\in S_r^t}l_i^t)^2 \leq (\vert S_l^t \vert + \vert S_r^t \vert) \sum_{i\in S_l^t \cup S_r^t} (l_i^t)^2$. Now, using $l_i^t=0\;\forall i\notin S_l^t\cup S_r^t$ and $\sum_{i\in S_l^t\cup S_r^t}l_i^t\leq (|S_l^t|+|S_r^t|)\sqrt{\sum_{i\in S_l^t\cup S_r^t}(l_i^t)^2}$, we get
\begin{align*}
\Delta_t  &
\geq \frac{\sum_{i=1}^{K-1}(l_i^t)^2 }{D}-4(|S_l^t|+|S_r^t|)\sqrt{\sum_{i\in S_l^t\cup S_r^t}(l_i^t)^2}\sqrt{\sum_{i\in S_l^t\cup S_r^t}(l_i^{t*})^2}\\
&\geq \frac{\sum_{i=1}^{K-1}(l_i^t)^2 }{D}-4(K-c-1)\sqrt{\sum_{i=1}^{K-1}(l_i^t)^2}\sqrt{\sum_{i=1}^{K-1}(l_i^{t*})^2}
\end{align*} 
Comparing the upper and lower bounds on $\sum_{t=1}^T\Delta_t$, we get
\begin{align*}
\sum_{t=1}^{T} \sum_{i =1}^{K-1} (l_{i}^{t})^{2}  &\leq \left(\Vert\mathbf{v}\Vert^{2}  + 4 K_1\sum_{t=1}^T\sqrt{\sum_{i=1}^{K-1}(l_i^t)^2}\sqrt{\sum_{i=1}^{K-1}(l_i^{t*})^2}\right)D 
\end{align*}
where $K_1=K-c-1$. Using Cauchy-Schwarz inequality, we get $\sum_{t=1}^T\sqrt{\sum_{i=1}^{K-1}(l_i^t)^2}\sqrt{\sum_{i=1}^{K-1}(l_i^{t*})^2}\leq L_T U_T$
where $L_T = \sqrt{\sum_{t=1}^{T} \sum_{i =1}^{K-1} (l_{i}^{t})^{2}}$ and $U_T =\sqrt{\sum_{t=1}^{T} \sum_{i =1}^{K-1} (l_{i}^{t*})^{2}}$. Thus, we get
$L_T^2 \leq D\left(\Vert\mathbf{v}\Vert^{2}  + 4 K_1 L_T U_T\right)$.
The upper bound on $L_T$ is obtained by the largest root of the polynomial $L_T^2 - 4 K_1DL_TU_T-D\Vert\mathbf{v}\Vert^{2}$ which is $2K_1DU_T+D\sqrt{4K_1^2U_T^2 + \Vert\mathbf{v}\Vert^{2}}$. Using the fact that $\sqrt{a+b}\leq \sqrt{a} + \sqrt{b}$, we get $L_T \leq D\Vert\mathbf{v}\Vert+ 4K_1DU_T$. Which means,
\begin{align*}
\sum_{t=1}^T\sum_{i=1}^{K-1}(l_i^t)^2 \leq D^2\left(\Vert\mathbf{v}\Vert+ 4K_1\sqrt{\sum_{t=1}^T\sum_{i=1}^{K-1}(l_i^{t*})^2}\right)^2
\end{align*}
\end{proof}
We know that $\sum_{i=1}^{K-1}(l_i)^2$ is an upper bound on the mean absolute error (MAE). Thus, $\sum_{t=1}^T\sum_{i=1}^{K-1}(l_i^t)^2$ is an upper bound on the number of mistakes in $T$ trials. Thus,  
$\sum_{t=1}^{T}L_I^{MAE}(\ww^t.\xx^t,\thetaa^t,y_l^t,y_r^t)
\leq D^2
\left(\Vert\mathbf{v}\Vert+ 4K_1 \sqrt{\sum_{t=1}^T\sum_{i=1}^{K-1}(l_i^{t*})^2}\right)^2$.
Note that this bound is same as given in \cite{Crammer:2001} when $c=0$ (exact label case). Now we will consider the ideal case.

\begin{corr}(Mistake Bound of PA in Ideal Case)
Let $(\mathbf{x}^{1}, y_{l}^{1}, y_{r}^{1}), \cdots,(\mathbf{x}^{T}, y_{l}^{T}, y_{r}^{T})$ be the sequence of examples presented to PA algorithm. Let $ {\bf v}^*=[{\mathbf{u}^*}'\;\;{{\bf b}^*}']$ be the parameters of an ideal classifier (where ${\bf u}^*\in \R^d $ and $ {\bf b}^*\in \R^{K-1} $) such that $\mathbf{u}^*.\xx^t -b_i^* \geq 1,\forall i \in [y_l^t-1],\forall t\in[T]$ and $\mathbf{u}^*.\xx^t - b_i^* \leq -1, \forall i \in \{y_r^t,\ldots,K-1\},\forall t\in[T]$. Then, the mistake bound of the PA algorithm is as follows.
\begin{align*}
\sum_{t=1}^{T}(l_i^t)^2 \leq \Vert \mathbf{v}^*\Vert^{2} \left(1+R^2(K-c-1)\right) 
\end{align*}
where $c=\min_{t\in [T]}(y_r^t-y_l^t)$ and $R^2=\max_{t\in[T]}\Vert\xx^t \Vert^2$.
\end{corr}
The proof of above can be easily seen by using the bound in Theorem~\ref{thm:regPA} and keeping $l_i^{t*}=0,\forall t\in [T],\forall i \in [K-1]$. 
Now we present the mistake bound for PA-I algorithm.

\begin{theorem}(Mistake Bound of PA-I in General Case)
\label{thm:PA-I}
Let $(\mathbf{x}^{1}, y_{l}^{1}, y_{r}^{1})$, $ \cdots$ $(\mathbf{x}^{T}, y_{l}^{T}, y_{r}^{T})$ be the sequence of examples. Let $c=\min_{t\in[T]}(y_r^t-y_l^t)$ and $ R^2=\max_{t\in[T]}\Vert \mathbf{x}^{t} \Vert^{2}$. Let ${\bf v} = [{\bf u}'\;\;{\bf b}']'$ be the parameters of an arbitrary ranking function ($\mathbf{u}\in \R^d $ and $ \mathbf{b}\in \R^{K-1}$). Then, the mistake bound of PA-I algorithm is given as
\begin{align*}
\sum_{t=1}^T\sum_{i=1}^{K-1}l_i^t &\leq \sum_{t=1}^T\sum_{i=1}^{K-1}l_i^{t*}+\sqrt{DT}\Vert {\bf v}\Vert
\end{align*}
where $D=1+2R^2(K-c-1)^2$.
\end{theorem}
The mistake bound proof for PA-I uses ideas from primal-dual techniques \cite{Shalev-Shwartz2007}. The proof is given in Appendix~\ref{app:PA-I}.
\begin{corr}(Mistake Bound of PA-I in Ideal Case)
\label{thm:PA-I}
Let $(\mathbf{x}^{1}, y_{l}^{1}, y_{r}^{1})$, $ \cdots$ $(\mathbf{x}^{T}, y_{l}^{T}, y_{r}^{T})$ be the sequence of examples. Let $c=\min_{t\in[T]}(y_r^t-y_l^t)$ and $ R^2=\max_{t\in[T]}\Vert \mathbf{x}^{t} \Vert^{2}$. Let there exists an ideal ranking function defined by ${\bf v}^* = [{{\bf u}*}'\;\;{{\bf b}^*}']'$ ($\mathbf{u}^*\in \R^d $ and $ \mathbf{b}^*\in \R^{K-1}$) such that $\mathbf{u}^*.\xx^t -b_i^* \geq 1,\forall i \in [y_l^t-1],\forall t\in[T]$ and $\mathbf{u}^*.\xx^t - b_i^* \leq -1, \forall i \in \{y_r^t,\ldots,K-1\},\forall t\in[T]$. Then, the mistake bound of PA-I algorithm is given as
\begin{align*}
\sum_{t=1}^T\sum_{i=1}^{K-1}l_i^t &\leq \sqrt{DT}\Vert {\bf v}\Vert
\end{align*}
where $D=1+2R^2(K-c-1)^2$.
\end{corr}
The proof of above Corollary is immediate from Theorem~\ref{thm:PA-I} by putting $l_i^{t*}=0,\;\forall t* \in [T],\;\forall i \in \{1,\ldots, y_l^{t*}-1,y_r^{t*},\ldots,K-1\}$. 

\begin{theorem}(Mistake Bound of PA-II in General Case) Let $(\mathbf{x}^{1}, y_{l}^{1}, y_{r}^{1})$, $ \cdots$ $(\mathbf{x}^{T}, y_{l}^{T}, y_{r}^{T})$ be the sequence of examples. Let $c=\min_{t\in[T]}(y_r^t-y_l^t)$ and $R^2=\max_{t\in[T]}\Vert \xx^t\Vert^2$. Let $ \mathbf{v}=[\mathbf{u}'\;\;\;\mathbf{b}']'$ $(\mathbf{u}\in \R^d,\; \mathbf{b}\in \R^{K-1})$ be the parameters of an arbitrary predictor. Then, for PA-II algorithm,
\begin{align*}
\sum_{t=1}^T\sum_{i=1}^{K-1}(l_i^t)^2 \leq  D\left(\Vert \mathbf{v} \Vert^2 +2C \sum_{t=1}^T\sum_{i =1}^{K-1} ( l_i^{t*})^2\right)
\end{align*}
where $D=1+\frac{1}{2C}+R^2(K-c-1)$.
\end{theorem}
\begin{proof}  
Let $\alpha=\frac{1}{\sqrt{2C}}$. Then,
\begin{align*}
& \Delta_t \geq  -(a^{t})^{2}\Vert \mathbf{x}^{t} \Vert^{2}  - \sum_{i \in S_{l}^{t}}(\lambda_i^t)^2- \sum_{i \in S_{r}^{t}}(\mu_i^t)^{2} + \sum_{i \in S_{l}^{t}} 2\lambda_i^t( l_{i}^{t} - l_i^{t*})\\
&\;\;\;\;+ \sum_{i \in S_{r}^{t}} 2\mu_i^t(l_{i}^{t} - l_i^{t*}) \\
&\geq  -(a^{t})^{2}\Vert \mathbf{x}^{t} \Vert^{2}  - \sum_{i \in S_{l}^{t}}(\lambda_i^t)^2- \sum_{i \in S_{r}^{t}}(\mu_i^t)^{2} + \sum_{i \in S_{l}^{t}} 2\lambda_i^t( l_{i}^{t} - l_i^{t*})\\
&\;\;\;\;+ \sum_{i \in S_{r}^{t}} 2\mu_i^t(l_{i}^{t} - l_i^{t*}) -\sum_{i\in S_l^t}(\alpha \lambda_i^t-\frac{l_i^{t*}}{\alpha})^2-\sum_{i\in S_r^t}(\alpha \mu_i^t-\frac{l_i^{t*}}{\alpha})^2\\
&=  -(a^{t})^{2}\Vert \mathbf{x}^{t} \Vert^{2}  - \left(1+\frac{1}{2C}\right)\left(\sum_{i \in S_{l}^{t}}(\lambda_i^t)^2 + \sum_{i \in S_{r}^{t}}(\mu_i^t)^{2}\right)\\
&-2C \left(\sum_{i \in S_{l}^{t}} ( l_i^{t*})^2+ \sum_{i \in S_{r}^{t}} ( l_i^{t*})^2\right)  +2\left(\sum_{i\in S_l^t} \lambda_i^tl_i^t + \sum_{i\in S_r^t}\mu_i^tl_i^{t}\right)\\
 &\geq  -(a^{t})^{2}\Vert \mathbf{x}^{t} \Vert^{2}  - \sum_{i \in S_{l}^{t}}\frac{\left(a^{t} \Vert \mathbf{x}^{t} \Vert^{2} - l_{i}^{t}\right)^2}{1 + \frac{1}{2C}}- \sum_{i \in S_{r}^{t}}\frac{\left(l_{i}^{t} + a^{t}\Vert \mathbf{x}^{t} \Vert^{2}\right)^2}{1 + \frac{1}{2C}}\\
 &\;\;\;\;-2C \sum_{i \in S_{l}^{t}\cup S_r^t} ( l_i^{t*})^2+2\left(\sum_{i\in S_l^t} \lambda_i^tl_i^t + \sum_{i\in S_r^t}\mu_i^tl_i^{t}\right)\\
&\geq \frac{\sum_{i\in S_l^t \cup S_r^t}(l_i^t)^2}{1+\frac{1}{2C}+R^2(K-c-1)} -2C \sum_{i \in S_{l}^{t}\cup S_r^t} ( l_i^{t*})^2
\end{align*}
We used $(\sum_{i\in S_l^t}l_i^t - \sum_{i\in S_r^t}l_i^t)^2 \leq (\vert S_l^t \vert + \vert S_r^t \vert) \sum_{i\in S_l^t \cup S_r^t} (l_i^t)^2$. Comparing the lower and upper bounds on $\sum_{t=1}^T \Delta t$. Let $D=1+\frac{1}{2C}+R^2(K-c-1)$, then
\begin{align*}
 \sum_{t=1}^T\sum_{i\in S_l^t \cup S_r^t}(l_i^t)^2 &\leq D\left(\Vert \mathbf{v} \Vert^2 +2C \sum_{t=1}^T\sum_{i \in S_{l}^{t}\cup S_r^t} ( l_i^{t*})^2\right) \\
& \leq D\left(\Vert \mathbf{v} \Vert^2 +2C \sum_{t=1}^T\sum_{i =1}^{K-1} ( l_i^{t*})^2\right)
\end{align*}
We know that $l_i^t = 0,\;\forall i \notin S_l^t\cup S_r^t$. Thus,
\begin{align*}
\sum_{t=1}^T\sum_{i=1}^{K-1}(l_i^t)^2 \leq  D\left(\Vert \mathbf{v} \Vert^2 +2C \sum_{t=1}^T\sum_{i =1}^{K-1} ( l_i^{t*})^2\right)
\end{align*}
\end{proof}

\begin{corr}(Mistake Bound of PA-II in Ideal Case) Let $(\mathbf{x}^{1}, y_{l}^{1}, y_{r}^{1})$, $ \cdots$ $(\mathbf{x}^{T}, y_{l}^{T}, y_{r}^{T})$ be the sequence of examples. Let $ \mathbf{v}^*=[{\mathbf{u}^*}'\;\;{\mathbf{b}^*}']'$ $(\mathbf{u}^*\in \R^d,\; \mathbf{b}^*\in \R^{K-1})$ be the parameters of an ideal predictor such that $\mathbf{u}^*.\xx^t - b_i^*\geq 1,\forall i \in [y_l^t-1],\forall t\in [T]$ and $\mathbf{u}^*.\xx^t -b_i^*\leq -1,\forall i \in \{y_r^t,\ldots,K-1\},\forall t\in[T]$. Let $c=\min_{t\in[T]}(y_r^t-y_l^t)$ and $R^2=\max_{t\in[T]}\Vert\xx^t\Vert^2$. Then, for PA-II algorithm,
\begin{align*}
\sum_{t=1}^T\sum_{i=1}^{K-1}(l_i^t)^2 \leq \left(1+\frac{1}{2C}+R^2(K-c-1)\right) \Vert \mathbf{v} \Vert^2  
\end{align*}
\end{corr}
\vspace{0.2in}
\section{Experiments}
In this section, we describe the experiments performed. 
\subsection{Datasets Used}
We perform experiments on following four datasets. The features in each of the dataset are normalized to zero mean and unit variance coordinate wise. 
\begin{itemize}
\item \textbf{California: } It contains information about the median house value in California from the 1990 census \cite{KelleyPace1997}. There are 20460 instances with 9 features. The aim is to predict median house value which ranges from 14999 to 500001. Since, this is a regression dataset, we created 5 intervals i.e, (1-100000),(100001-200000),(200001-300000),(300001-400000),(400001-500001) each representing a class.
\item \textbf{Abalone: }This dataset \cite{Dua:2017} has information related to the physical measurement of Abalone found in Australia. It has 4177 instances with 8 attribute. The aim is to predict the age of the Abalone using 'Rings' attribute which varies from 1-29. Due to the skewness of the distribution, we divided the 'Rings' attribute into 4 intervals as 1-7, 8-9, 10-12, 13-29.
\item \textbf{Parkinson Tele-monitoring: } This dataset \cite{Dua:2017} comprises of voice recordings of 42 patients at various stages of Parkinson's disease. There are 5875 instances with 22 features in the dataset. The target variable of this dataset is 'total\_UPDRS' for the instance which varies from 7 to 54.992 . We divided the 'total\_UPDRS' attribute into 4 classes i.e, 7-17, 18-27, 28-37, 38- 54.992.
\item \textbf{MSLR:} This dataset comprises of query-url pairs along with the relevance label obtained from the label set of commercial web search engine Microsoft Bing \cite{DBLP:journals/corr/QinL13}. The relevance label ranges from 0 (irrelevant) - 4 (perfectly relevant). We performed our experiment on MSLR-WEB10K in which we took 1 of the available 5 folds. There are 723412 instances divided in 5 classes, with each instance having 136 features.
\end{itemize}

\subsection{Generating Interval Labels}
 We generate interval labels as follows. Let $m$ be the fraction of interval labeled examples in the training data. We first randomly choose $m\%$ of the training data for which we generate interval labels. Then for each candidate example, we randomly assign one of the following interval label: $[y-1, y],\; [y, y+1],\; [y-1, y],\; [y-2, y],\; [y, y+2],\; [y-2, y+2]$ where $y$ is the actual label. We consider two different values of $m$, namely 50\% and 75\%. 

\subsection{Comparison Results with Other Approaches}
We compare the performance of proposed PA algorithm and its variants with two approaches. (a) PRank \cite{Crammer:2001} algorithm which is online ranking algorithm using the actual labels. (b) Multi-class 
Perceptron algorithm \cite{Crammer_Multi:2003} as ranking can also be viewed as multiclass classification (even though multiclass classification uses more parameters than ranking and ignores the ordering among class labels). 

\begin{figure*}[t]
\begin{center}
\begin{tabular}{ccc}
\includegraphics[scale=.16]{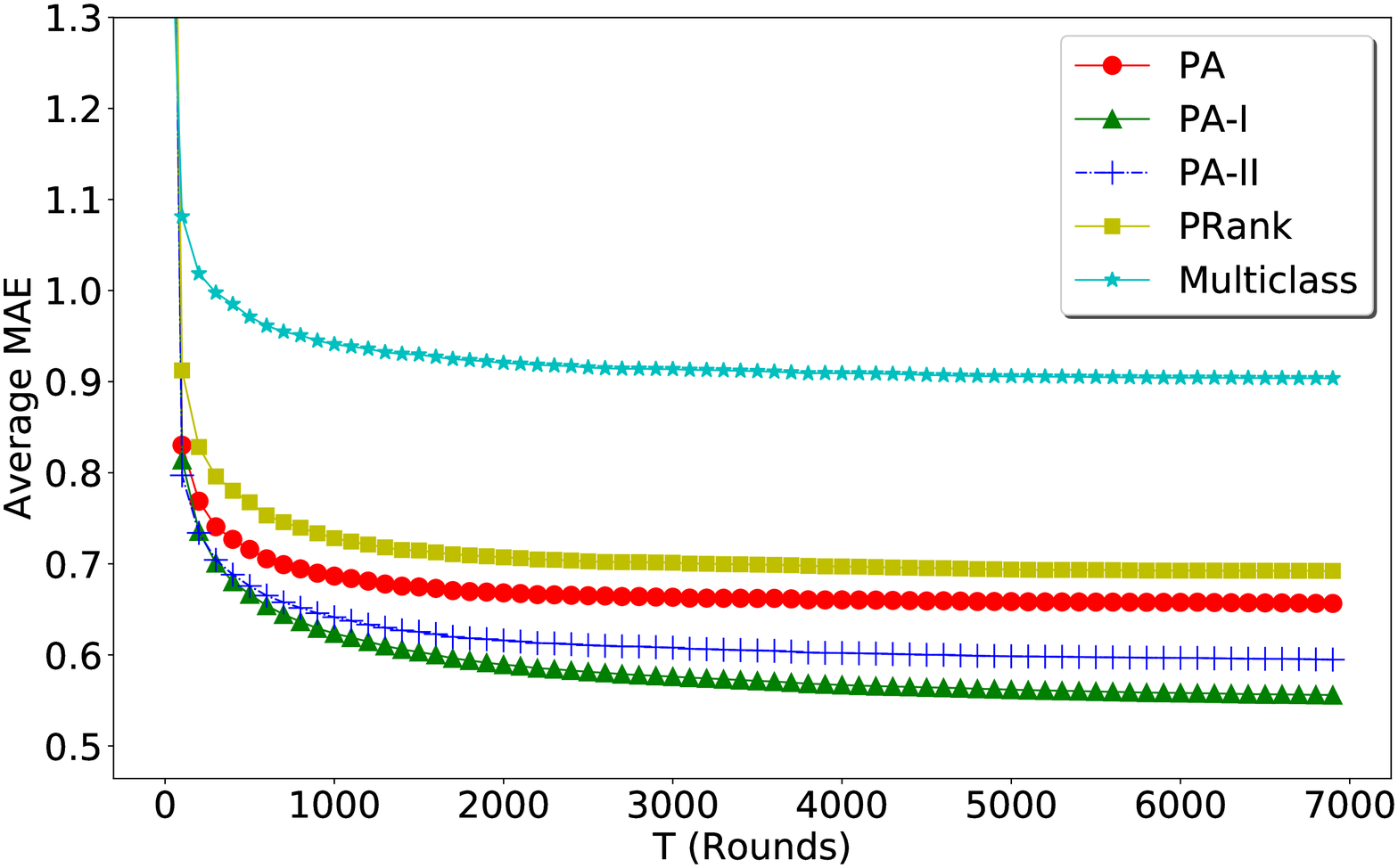} &
\includegraphics[scale=.16]{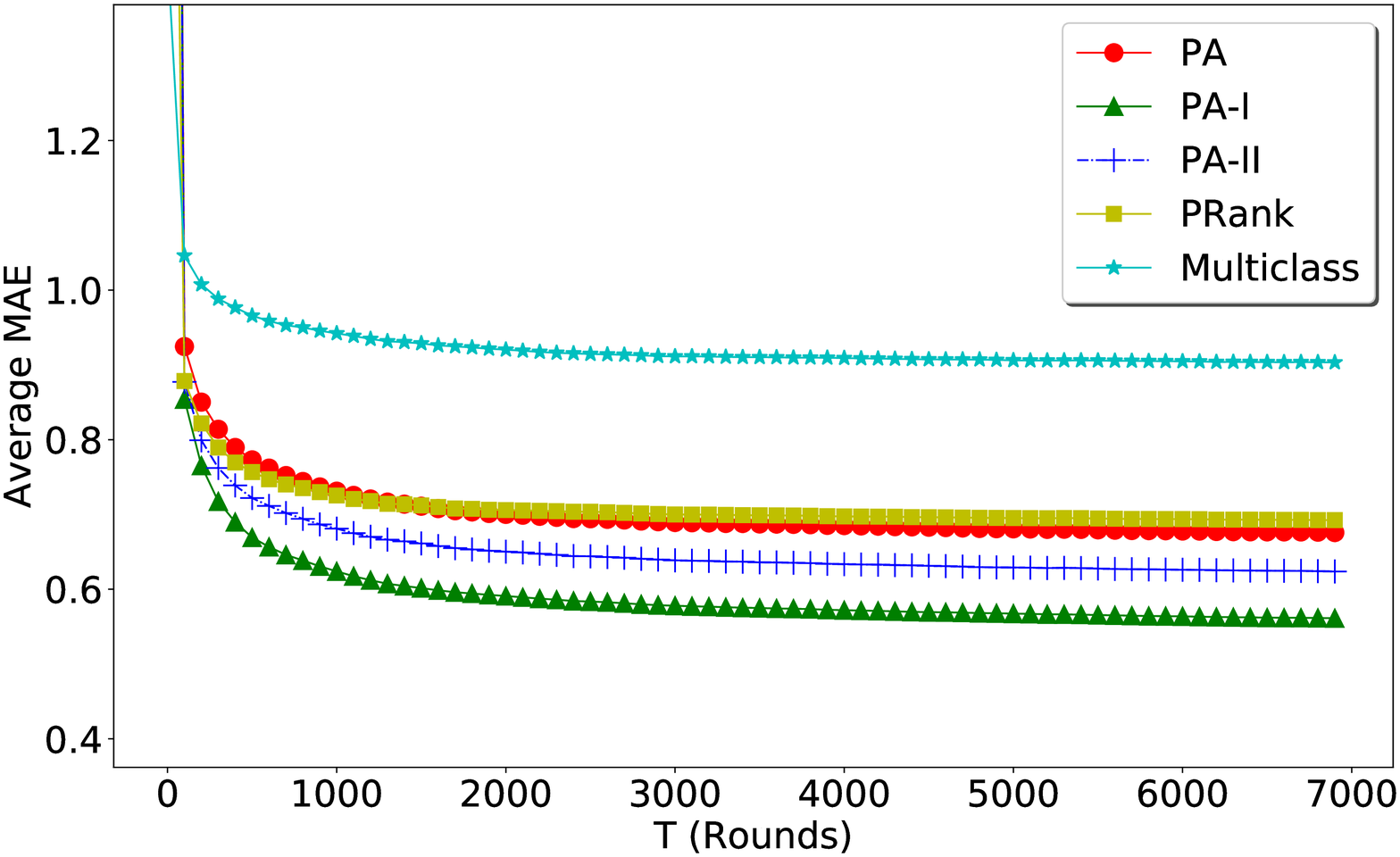} &
\includegraphics[scale=.16]{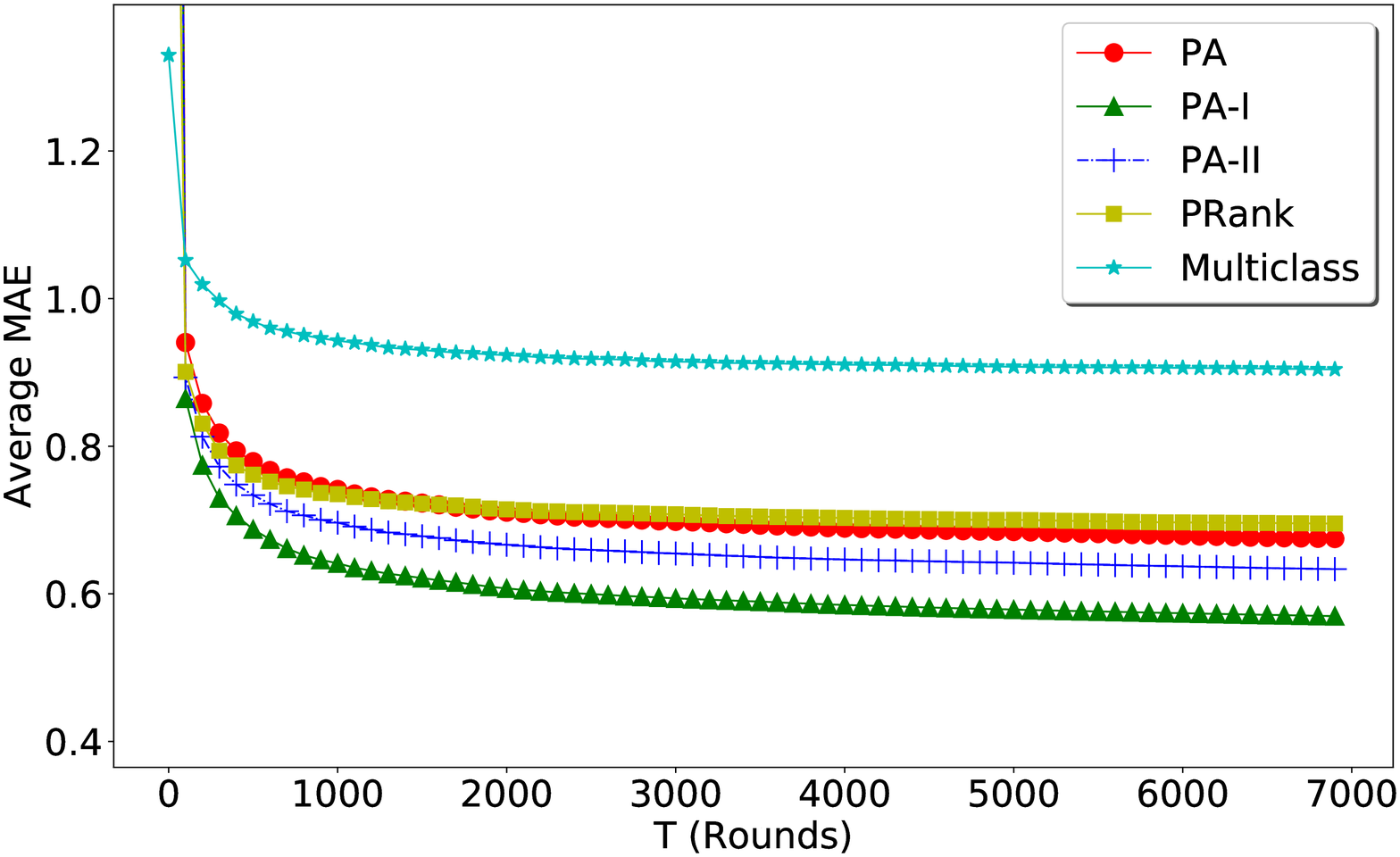} \\
California (exact labels) &
California (50\% partial labels) &
California (75\% partial labels) \\
\includegraphics[scale=.16]{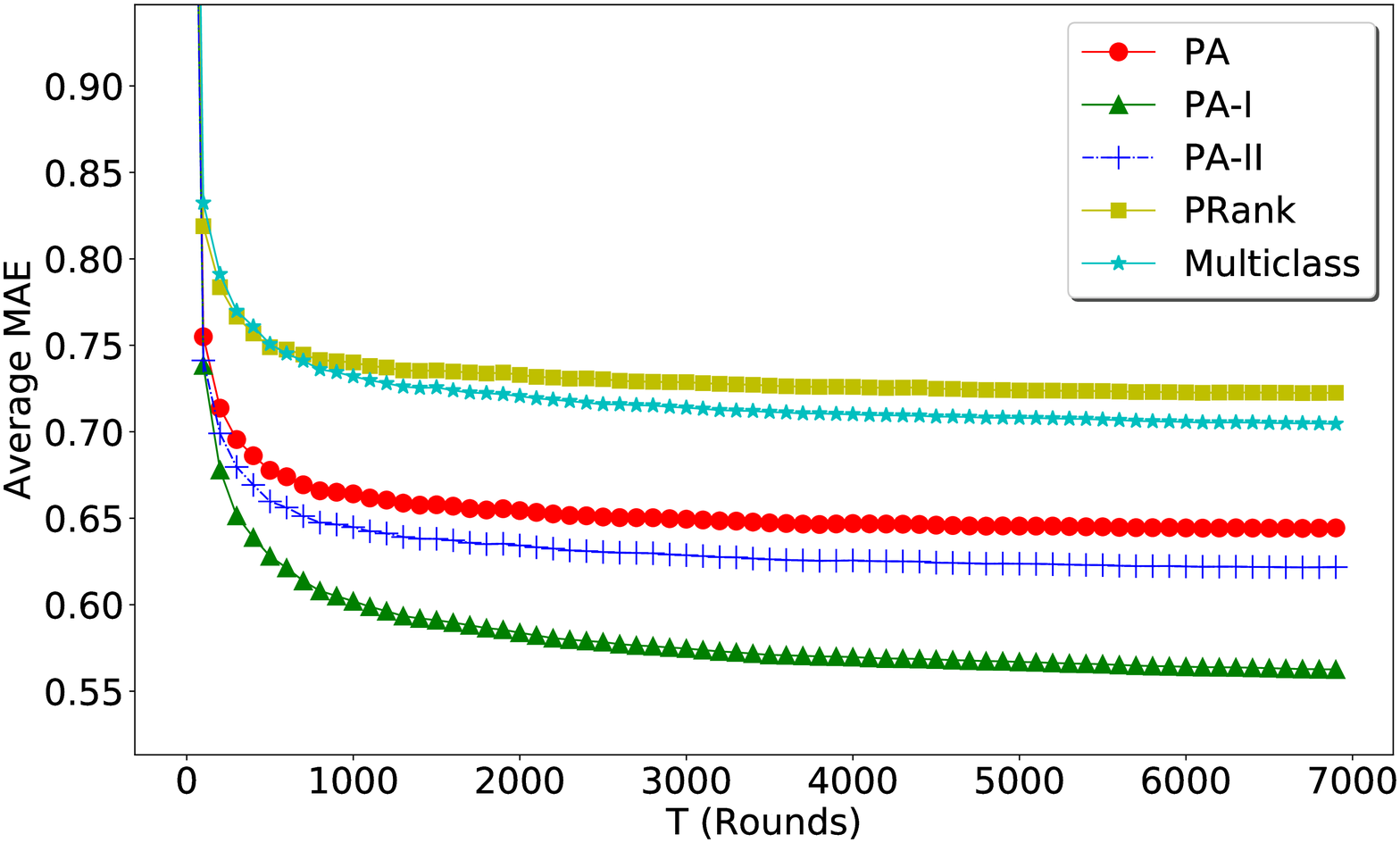} &
\includegraphics[scale=.16]{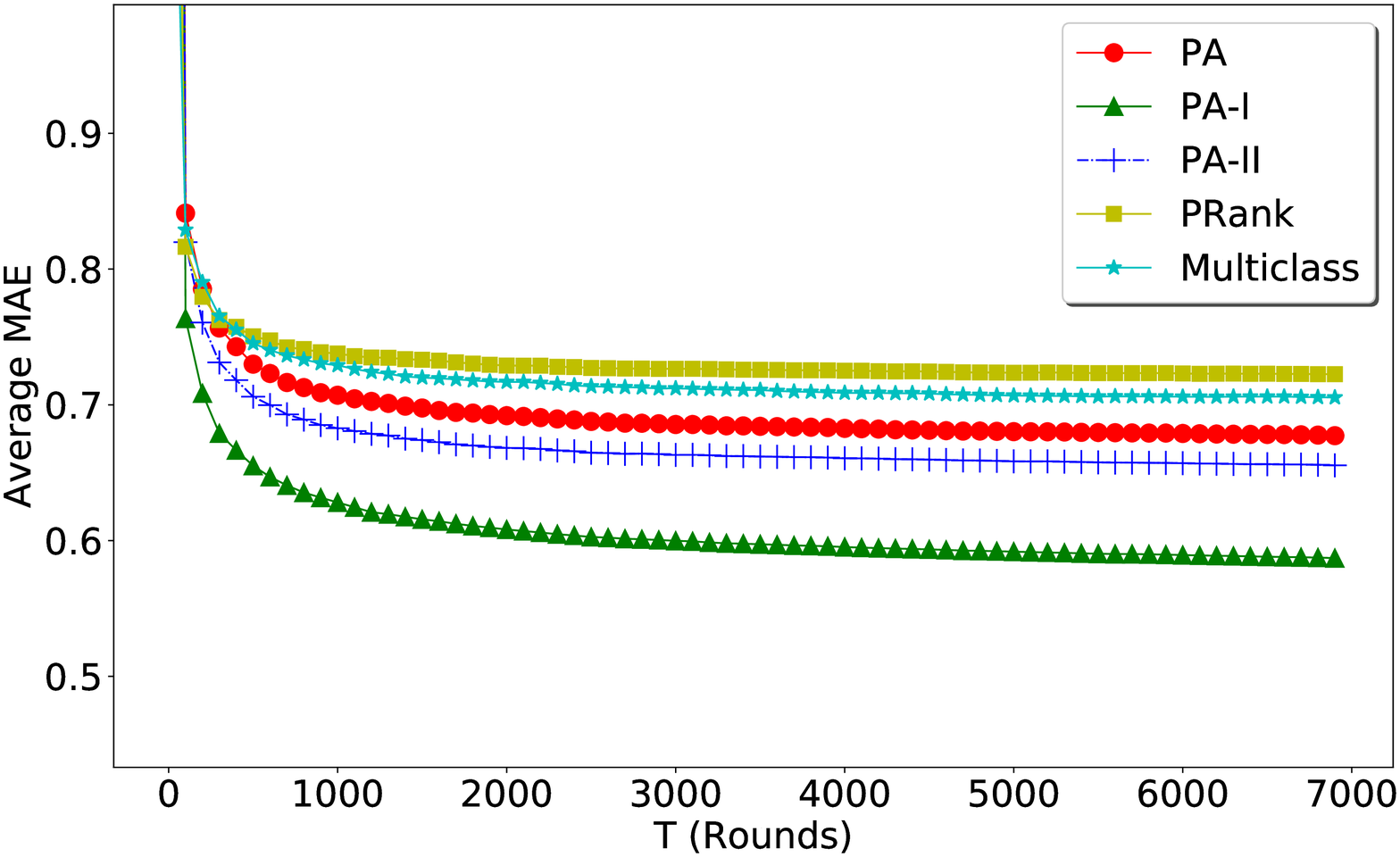} &
\includegraphics[scale=.16]{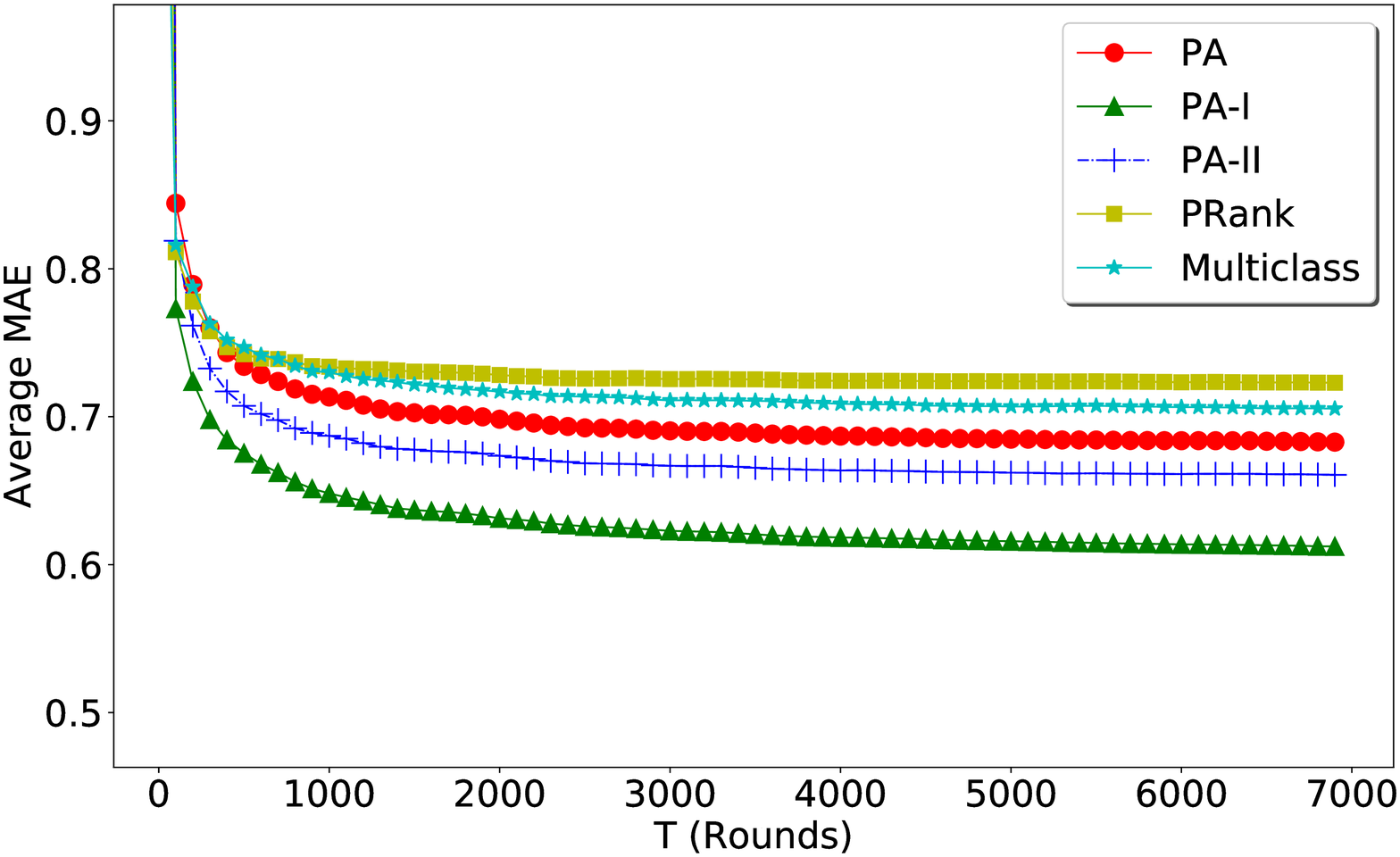} \\
Abalone (exact labels) &
Abalone (50\% partial labels) &
Abalone (75\% partial labels) \\
\includegraphics[scale=.16]{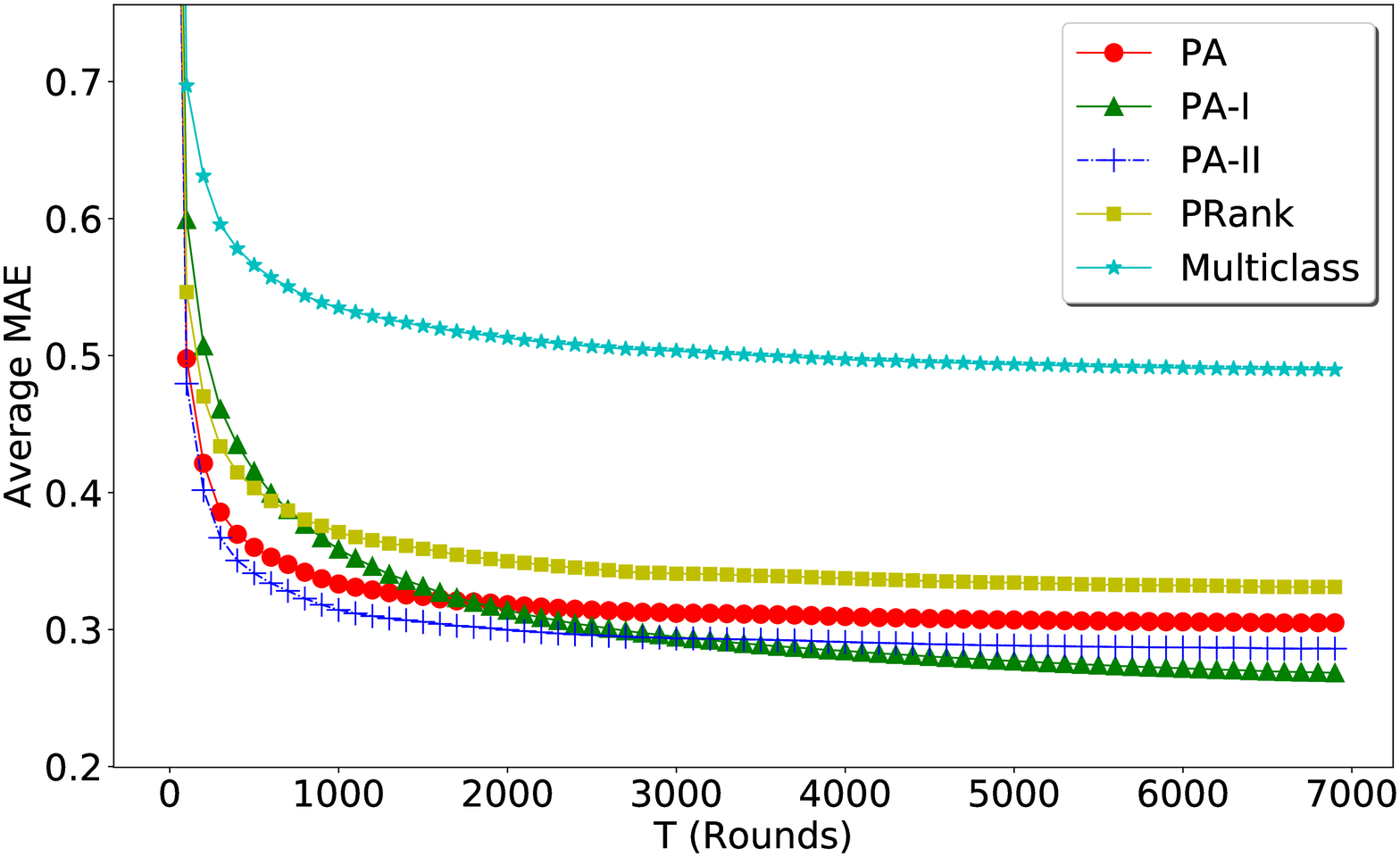} &
\includegraphics[scale=.16]{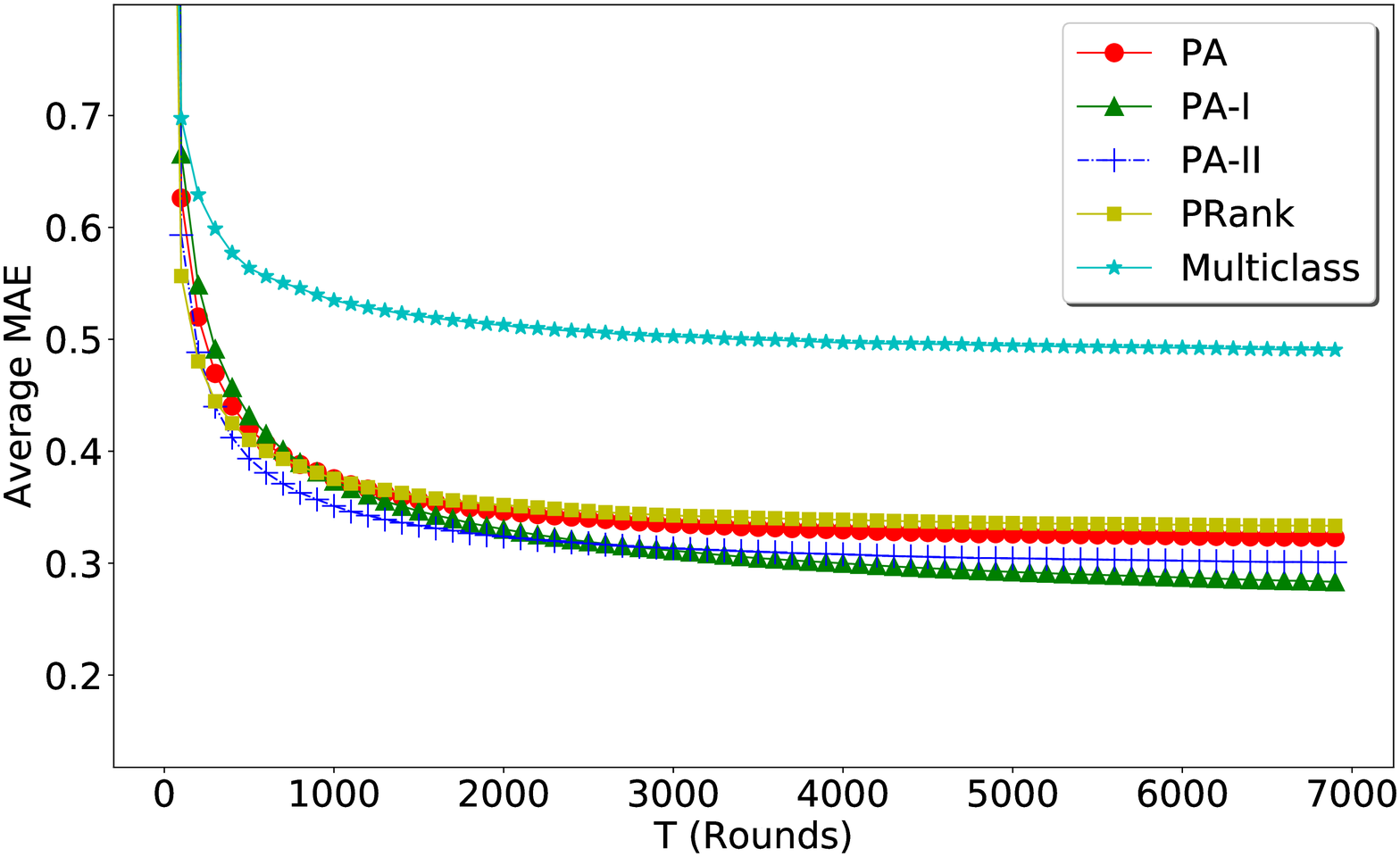} &
\includegraphics[scale=.16]{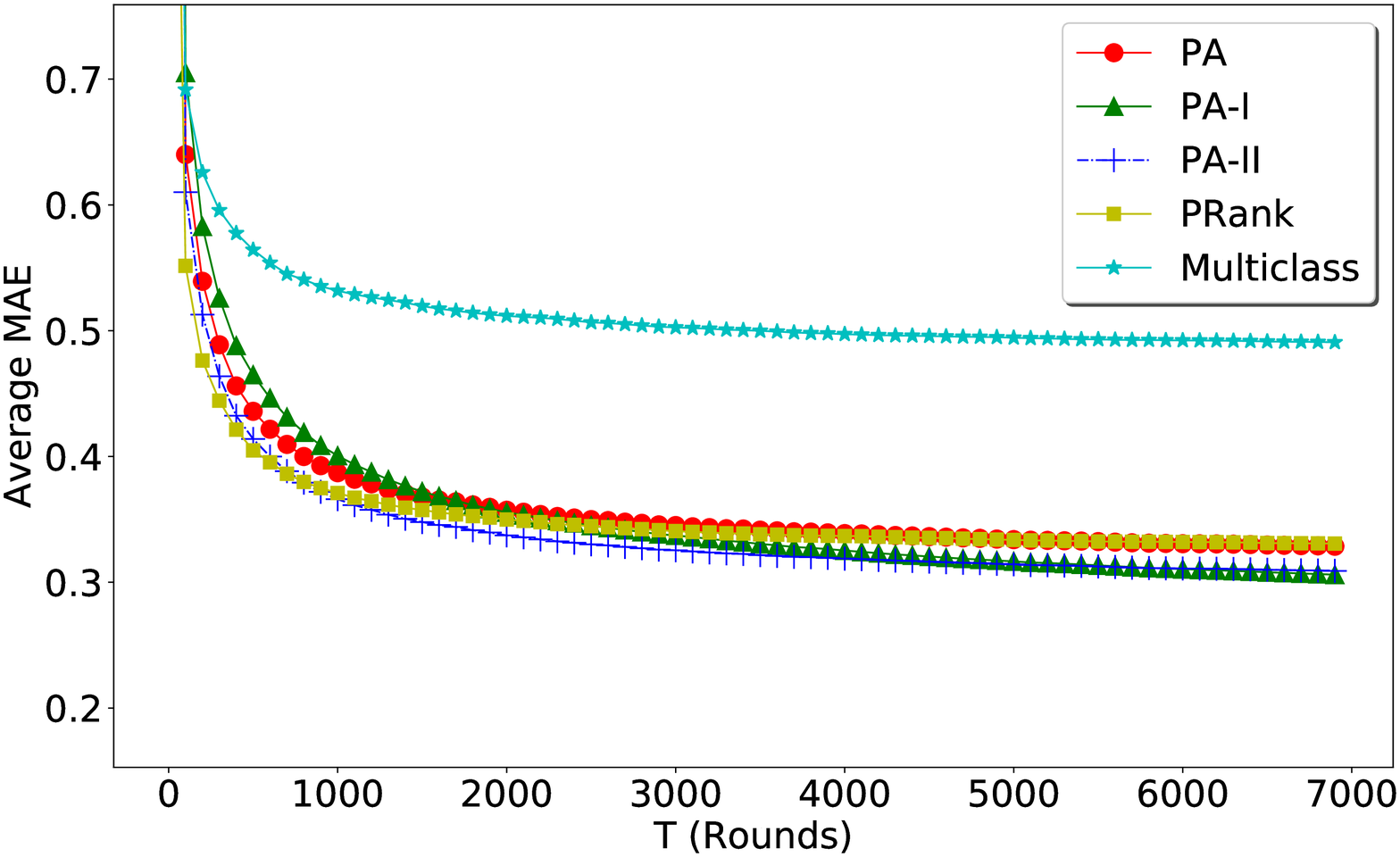}\\
Parkinson (exact labels) &
Parkinson (50\% partial labels) &
Parkinson (75\% partial labels) \\
\includegraphics[scale=.16]{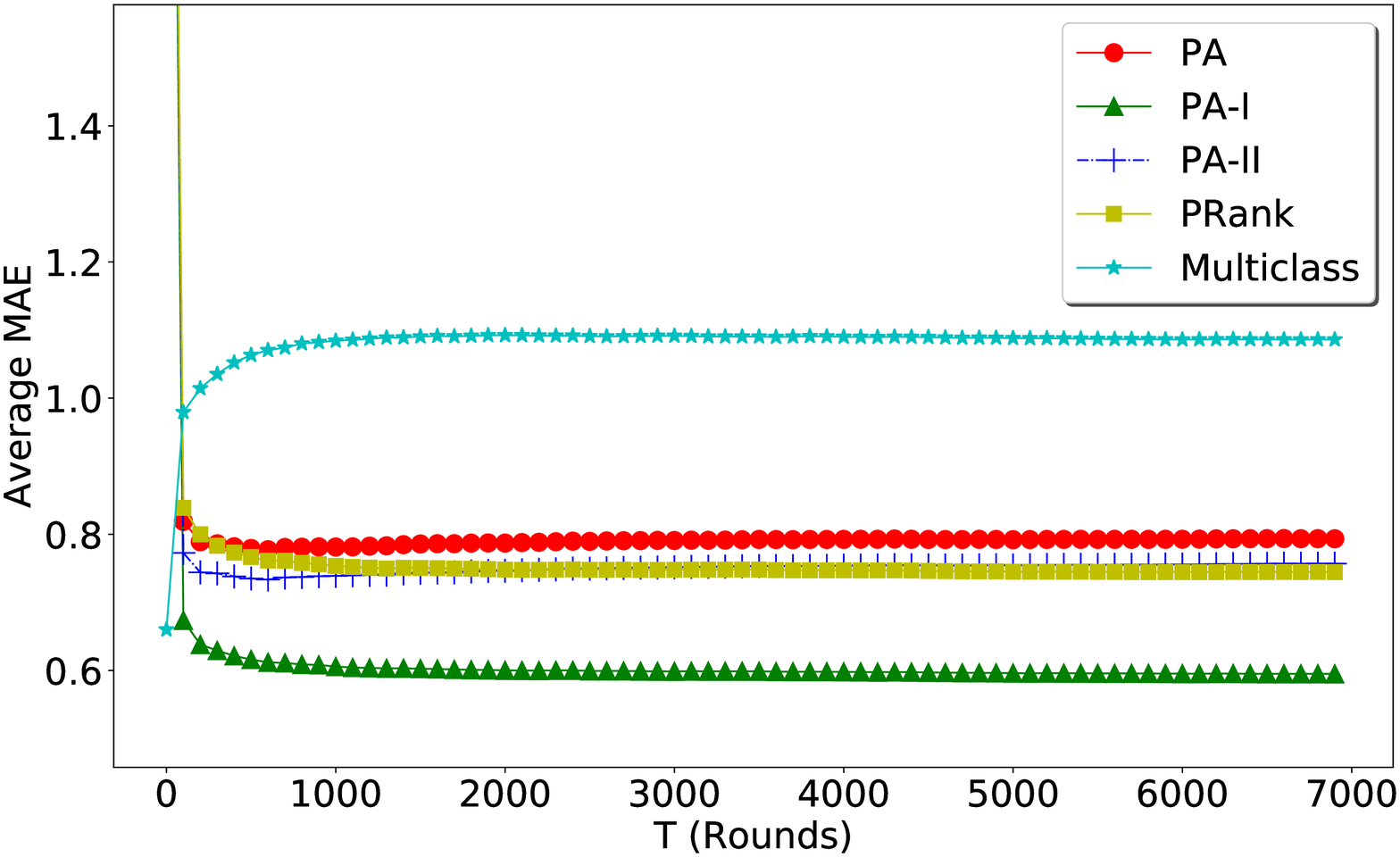} &
\includegraphics[scale=.16]{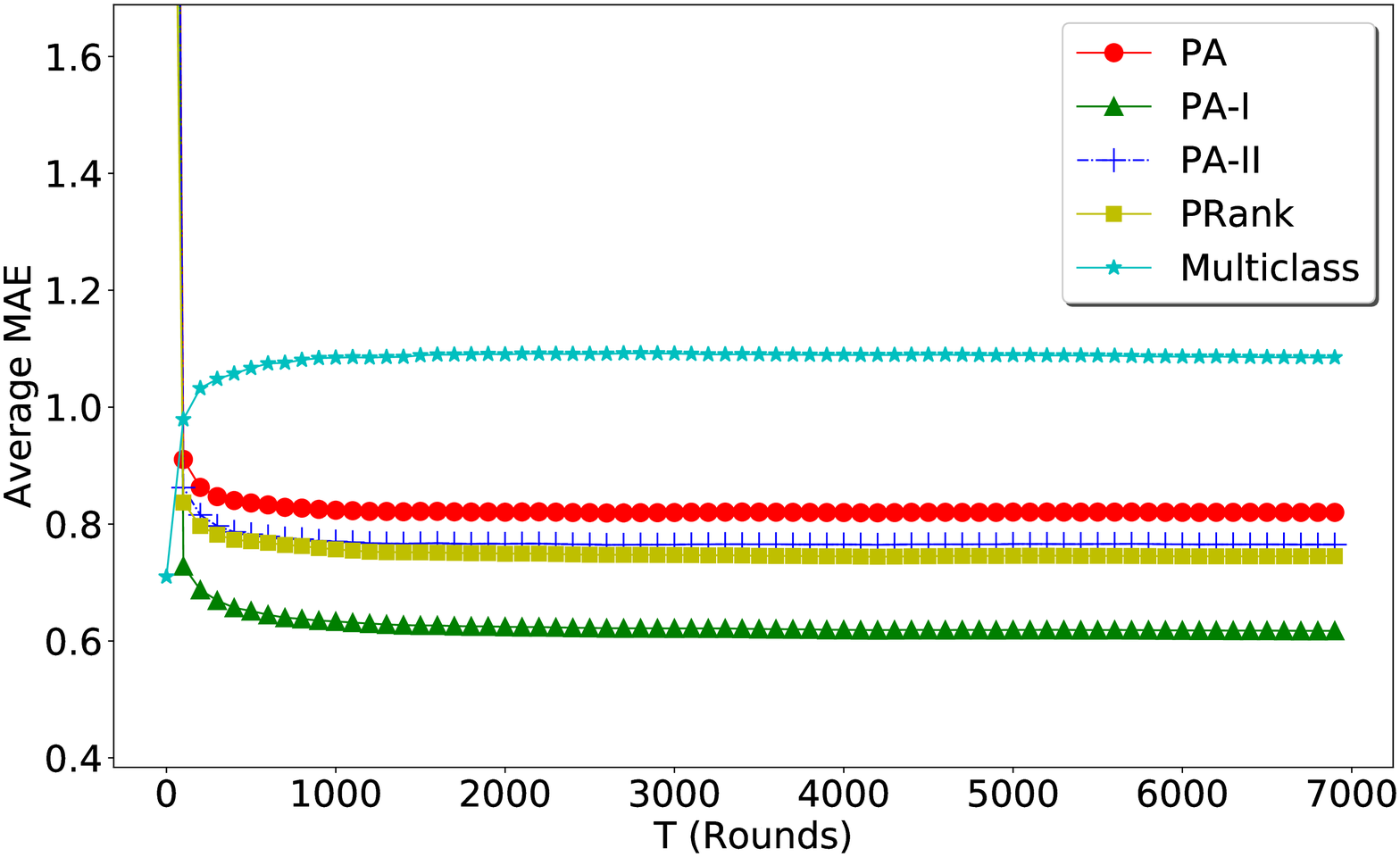} &
\includegraphics[scale=.16]{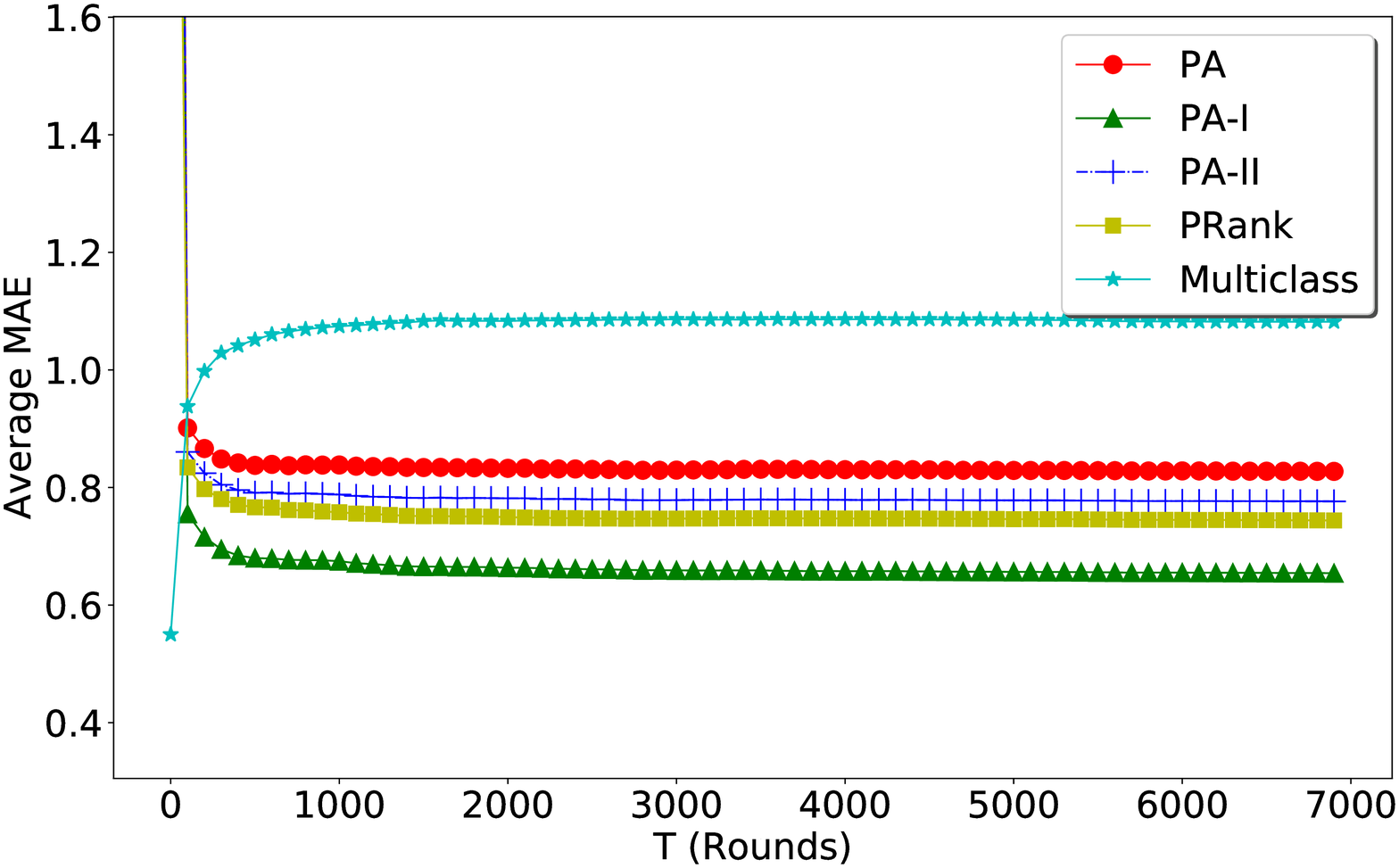}\\
MSLR (exact labels) &
MSLR (50\% partial labels) &
MSLR (75\% partial labels)
\end{tabular}
\caption{Comparison results of PA, PA-I and PA-II with MCP and PRank}
\label{fig:comparison}
\end{center}
\end{figure*}

For PRank and multi-class Perceptron (MCP) we used only the actual labels for training. For the proposed PA algorithms, we used interval labeled data during training. We took three different training sets for the proposed PA algorithms. First with 50\% interval labels, second with 75\% interval labels and third with actual (exact) labels. For our algorithms, we predicted the label for an example using the ranking function described in Eq.~(\ref{eq:ranking-function}).

We used the exact labels to compute the average $MAE$ (after every trial) for all the algorithms including the proposed PA algorithms. We find the average $MAE$ as $\frac{1}{t}\sum_{s=1}^t\vert \hat{y}^s - y^s\vert$ for $t=1\ldots 7000$. We repeat the process 100 times and average the instantaneous losses across the 100 runs.
Figure~\ref{fig:comparison}, we plot the average $MAE$ with respect to $t$. We observe the following.
\begin{itemize}
\item We see that for California and Abalone datasets, proposed PA algorithms (PA, PA-I and PA-II) trained using exact labels as well as using interval labels outperform the other algorithms.
\item For Parkinsons dataset, proposed PA variants outperform other approaches for exact labels case and 50\% interval labels case. For 75\% interval labels case, PA-I outperform PRank and MCP while PA and PA-II perform comparable to PRank.
\item For MSLR dataset, PA-I outperform both PRank and MCP for all 3 different kinds of labeling. Also, PA and PA-II always outperform MCP. PA-II performs comparable to PRank.
\end{itemize}
Thus, we see that the proposed PA algorithms perform better compared to PRank and MCP.

\begin{figure}
\begin{center}
\begin{tabular}{c}
\includegraphics[scale=.23]{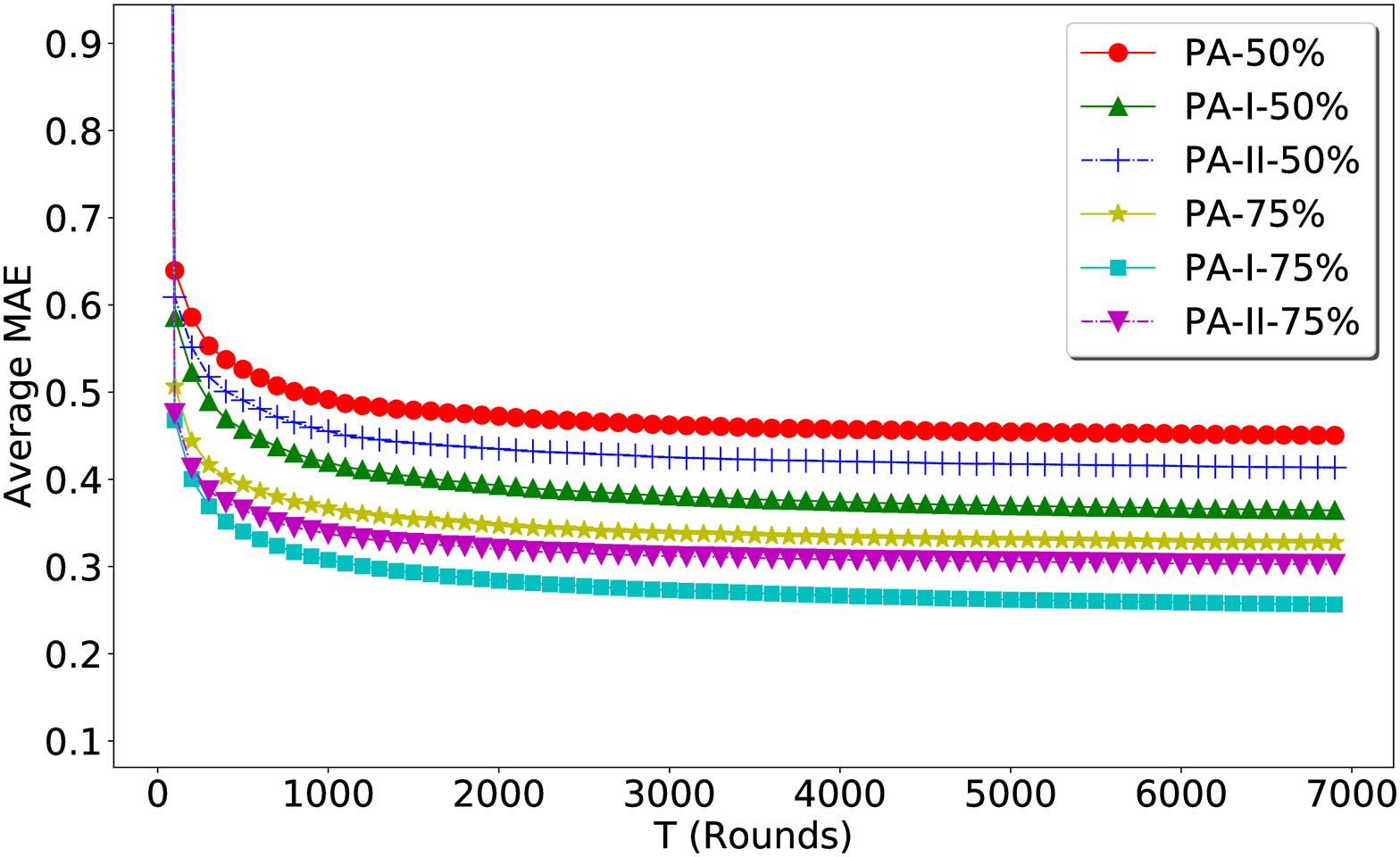} \\
California ( 50 \% and 75\% partial labels) \\
\includegraphics[scale=.23]{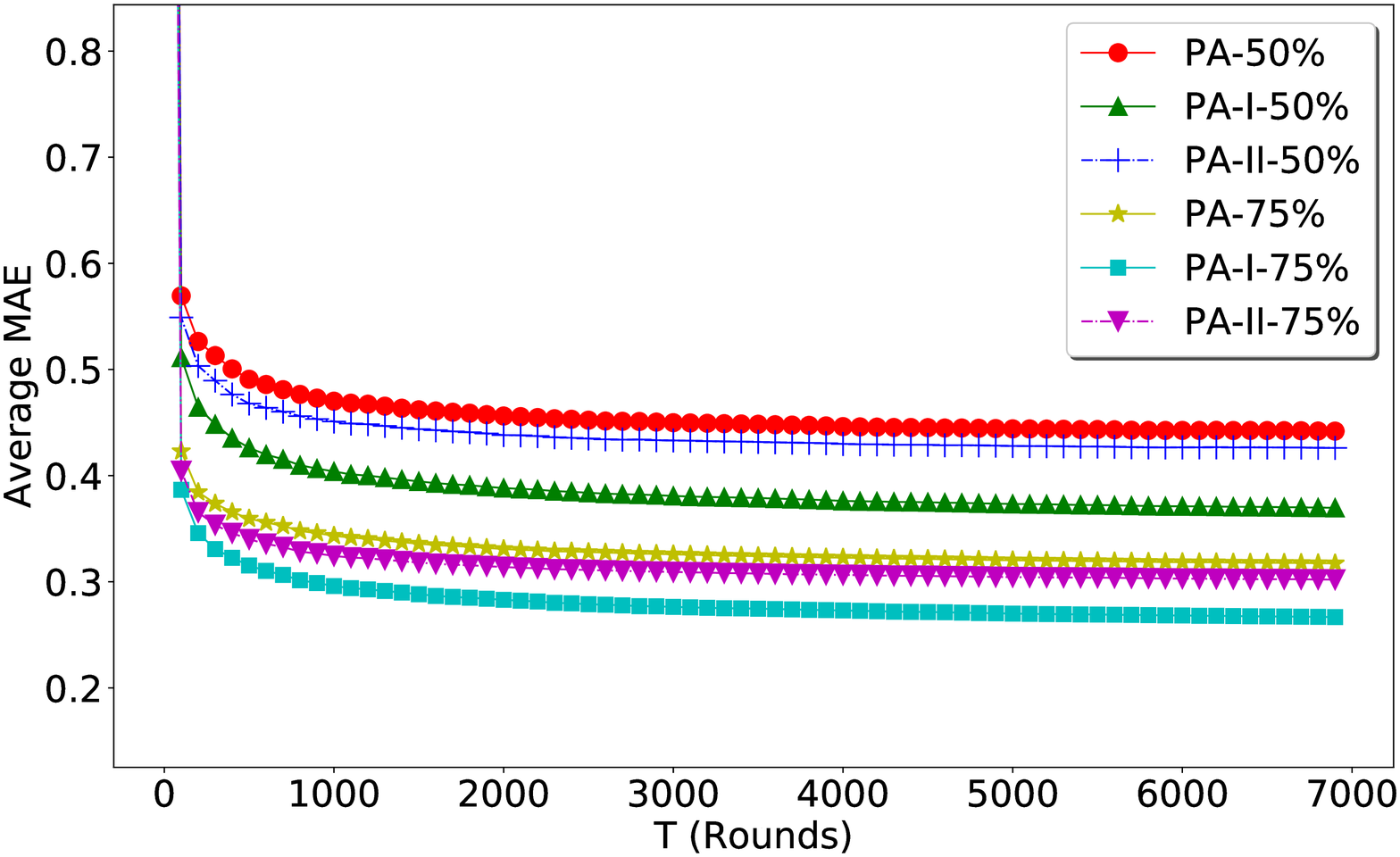} \\
Abalone (50 \% and 75\% partial labels) \\
\includegraphics[scale=.23]{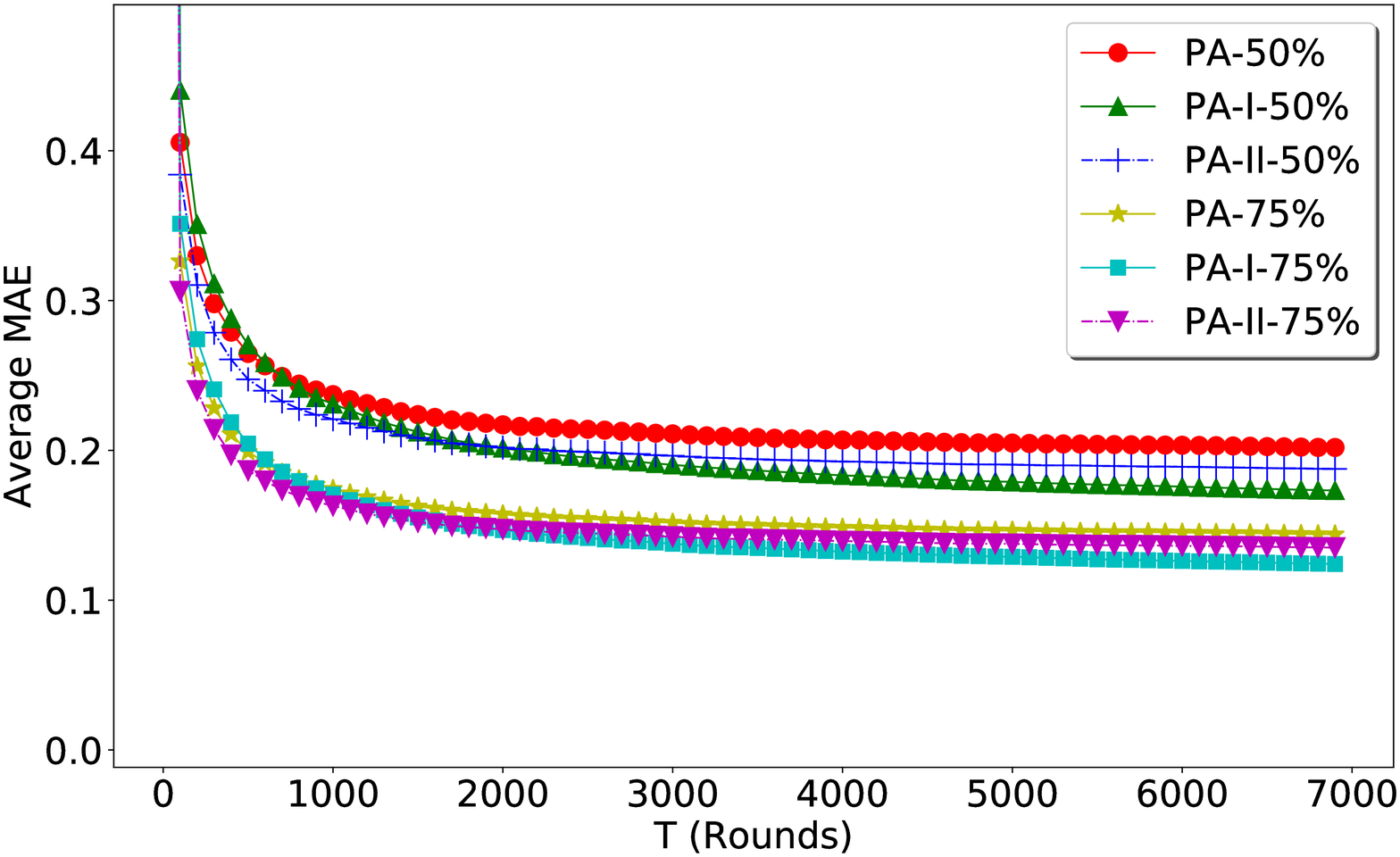}\\
Parkinson (50 \% and 75\% partial labels) \\
\includegraphics[scale=.23]{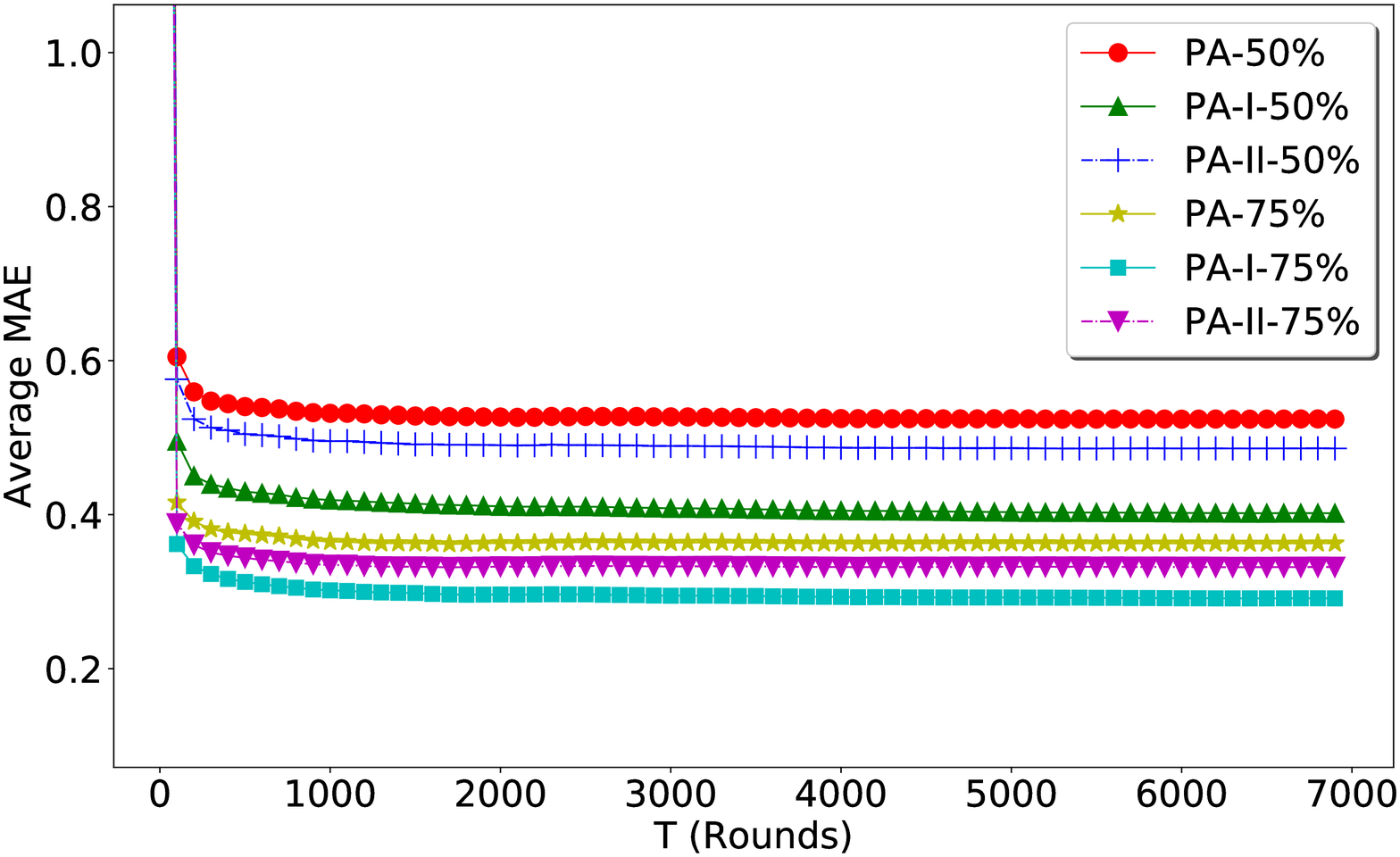}\\
MSLR (50 \% and 75\% partial labels)
\end{tabular}
\caption{Experiment: Varying the fraction of interval labels. Average MAE performance decreases by increasing the fraction of partial labels in the training set. MAE is computed considering partial labels.}
\label{fig:Vary-Frac-Int}
\end{center}
\end{figure}

\subsection{Results with Varying the Fraction of Labels}
We vary the fraction of partial labels (50\% and 75\%). We compute the average $MAE$ after every trial with the same interval label used for updating the hypothesis. We repeat the process 100 times and average the instantaneous losses across the 100 runs. The results are shown in Figure~\ref{fig:Vary-Frac-Int}. We see that for all the datasets the average $MAE$ decreases faster as compared to the number of trial $T$. Also, the average $MAE$
decreases with the increase in the fraction of interval labels. This happens because the
allowed range for predicted rank is more when we use interval labels for computing $MAE$.

\section{Conclusions}
In this paper we proposed online passive-aggressive algorithms for learning to rank. A very important feature of the proposed algorithms is that it also works for instances having interval labels. This becomes useful when annotators are unable to give a fixed label to an instance. We presented three variants of PA algorithms namely PA, PA-I and PA-II. We find the exact solution of the optimization problem at every trial. Our method is based on finding the support classes $S_{l}^{t}$ and $S_{r}^{t}$ at each instant using the SCA algorithms. These sets describe the thresholds to be updated at a trial. Advantage of our method is that the ordering of the thresholds is maintained implicitly and this has been proved theoretically in this paper. In addition to this, we have also given mistake bounds on all the three variants of the algorithm. Practical experiments show that our proposed algorithms perform better than the other algorithms (PRank and Multiclass Percpetron) even when we train our algorithms using interval labels. 

\begin{appendices}
\section{Proof of Theorem~\ref{thm:PA-I}: Mistake Bound Proof for PA-I in General Case}\label{app:PA-I}
\begin{proof}
We use the primal-dual framework proposed in \cite{Shalev-Shwartz2007,10.1007/11776420_32} to get the bound. In that framework, online learning is posed as a task of incrementally increasing the dual objective function. The dual optimization problem (${\cal D}$) of the regularized risk under $L_I^{IMC}$ (considering all $T$ examples) is  
\begin{align*}
&\underset{\alphaa^1\ldots\alphaa^T}{\max}\;\sum_{t=1}^T(\sum_{j=1}^{y_l^t}\lambda_j^t+\sum_{j=y_r^i}^{K-1}\mu_j^t)-\frac{1}{2}\Vert \sum_{t=1}^T(\sum_{j=1}^{y_l^t}\lambda_j^t-\sum_{j=y_r^i}^{K-1}\mu_j^t)\xx^t\Vert^2\\
&\;\;\;\;-\frac{1}{2}\sum_{j=1}^{K-1}\left(\sum_{t=1}^T\left(\mu_j^t\I_{\{j\geq y_r^t\}}-\lambda_j^t\I_{\{j\leq y_l^t-1\}}\right)\right)^2\\
&s.t. \;0\leq \lambda_j^t\leq C,\;t\in [T],\;j=1\ldots y_l^t-1\\
&\;\;\;\;\;\;0\leq \mu_j^t\leq C,\;t\in [T],\;j=y_r^t\ldots K-1
\end{align*}
where $\alphaa^t=[\lambda_1^t\;\ldots\;\lambda_{y_l^t-1}\;0\;\ldots\;0\;\mu_{y_r^t}\;\ldots\;\mu_{K-1}^t]\in \R^{K-1}$. Let $\Omega=(\alphaa^1,\ldots,\alphaa^T)$. 
PA-I can be viewed as finding a sequence of $\Omega^1,\ldots,\Omega^{T+1}$ where $\Omega^{t+1}=(\alphaa^1_{t+1},\ldots,\alphaa^T_{t+1})$ is the maximizer of the following problem.
\begin{align*}
\underset{\Omega}{\max}\;\;{\cal D}(\Omega)\qquad s.t.\;\; \alphaa^s={\bf 0},\;\;\forall s> t 
\end{align*}
PA-I updates are as follows. $\alphaa_{t+1}^i=
\alphaa_t^i,\;\forall i\neq t$. $\alphaa^t_{t+1}=[\lambda_1^t\;\ldots\;\lambda_{y_l^t-1}^t\;0\;\ldots\;0\;\mu_{y_r^t}^{t}\;\ldots\;\mu_{K-1}^t]$ where $\lambda_i^t=\min(C,l_i^t-a^t\Vert \xx^t\Vert^2),\;i=1\ldots y_l^t-1$ and $\mu_i^t=\min(C,l_i^t+a^t\Vert\xx^t \Vert^2),\;i=y_r^t\ldots K-1$.
Increment in ${\cal D}$ after trial $t$ is,
\begin{align*}
&\nonumber {\cal D}(\Omega^{t+1})-{\cal D}(\Omega^{t})= -\frac{1}{2} (\sum_{i=1}^{y_l^t-1}\lambda_i^t-\sum_{i=y_r^t}^{K-1}\mu_i^t)^2\Vert\xx^t\Vert^2-\frac{1}{2}\sum_{i=1}^{y_l^t-1}(\lambda_i^t)^2\\
&-\frac{1}{2}\sum_{i=y_r^t}^{K-1}(\mu_i^t)^2+ \sum_{i=1}^{y_l^t-1}\lambda_i^t(1-\ww^t.\xx^t+\theta_i^t)+ \sum_{i=y_r^t}^{K-1}\mu_i^t(1+\ww^t.\xx^t-\theta_i^t)
\end{align*}
where $\theta_i^t=\sum_{s=1}^{t-1}\left(\mu_i^s\I_{\{i\geq y_r^s\}}-\lambda_i^s\I_{\{i\leq y_l^s-1\}}\right),\;i\in[K-1]$ and $\ww^t=\sum_{s=1}^{t-1}a^s\xx^s$. Note that $\lambda_i^t>0,\;i\in S_l^t$ and $\mu_i^t>0,\;i\in S_r^t$. Using $a^t=\sum_{i\in S_l^t}\lambda_i^t-\sum_{i\in S_r^t}\mu_i^t$, we get,
\begin{align}
&\nonumber {\cal D}(\Omega^{t+1})-{\cal D}(\Omega^{t})= 
\sum_{i\in S_l^t}\lambda_i^t(l_i^t-a^t\Vert \xx^t\Vert^2-\frac{\lambda_i^t}{2})\\
\nonumber &+\sum_{i\in S_r^t}\mu_i^t(l_i^t+a^t\Vert \xx^t\Vert^2-\frac{\mu_i^t}{2})+\frac{1}{2}a^t\Vert\xx^t\Vert^2\left(\sum_{i\in S_l^t}\lambda_i^t-\sum_{i\in S_r^t}\mu_i^t\right)\\
 &\geq C[\sum_{i\in S_l^t}\gamma(l_i^t-a^t\Vert \xx^t\Vert^2)+\sum_{i\in S_r^t}\gamma(l_i^t+a^t\Vert \xx^t\Vert^2)]
\label{eq:dual-diff}
\end{align}
where $\gamma(z)=\frac{1}{C}\left(\min(z,C)\left(z-\frac{1}{2}\min(z,C)\right)\right)$ \cite{Shalev-Shwartz2007}. Note that ${\cal D}(\Omega^0)=0$. Summing Eq.~(\ref{eq:dual-diff}) from $t=1$ to $T$, we get
\begin{align*}
&{\cal D}(\Omega^{T+1})=\sum_{t=1}^T\left({\cal D}(\Omega^{t+1})-{\cal D}(\Omega^t)\right)\\
&\geq C\sum_{t=1}^T\left(\sum_{i\in S_l^t}\gamma(l_i^t-a^t\Vert \xx^t\Vert^2)+\sum_{i\in S_r^t}\gamma(l_i^t+a^t\Vert \xx^t\Vert^2)\right)
\end{align*}
Note that $\gamma(.)$ is a convex function \cite{Shalev-Shwartz2007}. Thus,
\begin{align*}
&{\cal D}(\Omega^{T+1})\geq C\sum_{t=1}^T\left(\sum_{i\in S_l^t}\gamma(l_i^t-a^t\Vert \xx^t\Vert^2)+\sum_{i\in S_r^t}\gamma(l_i^t+a^t\Vert \xx^t\Vert^2)\right)\\
&\geq CT\gamma\left(\frac{1}{T}\sum_{t=1}^T\left(\sum_{i\in S_l^t}(l_i^t-a^t\Vert \xx^t\Vert^2)+\sum_{i\in S_r^t}(l_i^t+a^t\Vert \xx^t\Vert^2)\right)\right)
\end{align*}
From the weak duality, we get the following.
\begin{align*}
{\cal D}(\Omega^{T+1})\leq \frac{1}{2}\left(\Vert\ww\Vert^2+\Vert \thetaa\Vert^2\right)+C\sum_{t=1}^T\left(\sum_{i=1}^{y_l^t-1}l_i^{t*}+\sum_{i=y_r^t}^{K-1}l_i^{t*}\right)
\end{align*}
Comparing the upper bound and the lower bound on ${\cal D}(\Omega^{T+1})$, we get
\begin{align*}
&\frac{1}{T}\sum_{t=1}^T\left(\sum_{i\in S_l^t}(l_i^t-a^t\Vert \xx^t\Vert^2)+\sum_{i\in S_r^t}(l_i^t+a^t\Vert \xx^t\Vert^2)\right)\\
&\leq \gamma^{-1}\left(\frac{1}{2CT}\left(\Vert\ww\Vert^2+\Vert \thetaa\Vert^2\right)+\frac{1}{T}\sum_{t=1}^T\left(\sum_{i=1}^{y_l^t-1}l_i^{t*}+\sum_{i=y_r^t}^{K-1}l_i^{t*}\right)\right)
\end{align*}
We note that 
\begin{align}
\nonumber &\frac{1}{T}\sum_{t=1}^T\left(\sum_{i\in S_l^t}(l_i^t-a^t\Vert \xx^t\Vert^2)+\sum_{i\in S_r^t}(l_i^t+a^t\Vert \xx^t\Vert^2)\right)\\
\nonumber &\geq 
\frac{1}{T}\sum_{t=1}^T\sum_{i\in S_l^t\cup S_r^t}l_i^t -\frac{1}{T}\sum_{t=1}^Ta^t\Vert\xx^t\Vert^2(|S_l^t|+|S_r^t|)\\
&\geq \frac{1}{T}\sum_{t=1}^T\sum_{i=1}^{K-1}l_i^t - CR^2(K-c-1)^2
\label{eq:lb}
\end{align}
where we used the fact that $\Vert\xx^t\Vert^2\leq R^2,\;\forall t\in [T]$, $|S_l^t|+|S_r^t|\leq K-c-1,\;\forall t\in [T]$ and $a^t\leq C(K-c-1),\;\forall t\in [T]$. From \cite{Shalev-Shwartz2007}, we know that $\gamma^{-1}(z) \leq  z+\frac{1}{2}C$. Thus,
\begin{align}
\nonumber & \gamma^{-1}\left(\frac{1}{2CT}\left(\Vert\ww\Vert^2+\Vert \thetaa\Vert^2\right)+\frac{1}{T}\sum_{t=1}^T\left(\sum_{i=1}^{y_l^t-1}l_i^{t*}+\sum_{i=y_r^t}^{K-1}l_i^{t*}\right)\right)\\
 & \leq \frac{1}{2CT}\left(\Vert\ww\Vert^2+\Vert \thetaa\Vert^2\right)+\frac{1}{T}\sum_{t=1}^T\left(\sum_{i=1}^{y_l^t-1}l_i^{t*}+\sum_{i=y_r^t}^{K-1}l_i^{t*}\right)+\frac{C}{2}
\label{eq:ub}
\end{align}
Using Eq.~(\ref{eq:lb}) and (\ref{eq:ub}), we get
\begin{align*}
\sum_{t=1}^T\sum_{i=1}^{K-1}l_i^t &\leq \frac{1}{2C}\Vert {\bf v}\Vert^2+\sum_{t=1}^T\sum_{i=1}^{K-1}l_i^{t*}+CT[\frac{1}{2}+R^2(K-c-1)^2]
\end{align*}
We use $C=\frac{\Vert {\bf v}\Vert}{\sqrt{T(1+2R^2(K-c-1)^2)}}$ as it minimizes the upper bound. Using that, we get
\begin{align*}
\sum_{t=1}^T\sum_{i=1}^{K-1}l_i^t &\leq \sum_{t=1}^T\sum_{i=1}^{K-1}l_i^{t*}+\sqrt{T\left(1+2R^2(K-c-1)^2\right)}\Vert {\bf v}\Vert
\end{align*}
\end{proof}
\end{appendices}

\bibliography{nips_2018_bib}
\bibliographystyle{plain}

\end{document}